\documentclass{article}

\usepackage[nonatbib,preprint]{neurips_2023}

\usepackage{multirow}
\usepackage[utf8]{inputenc} 
\usepackage[T1]{fontenc}    
\usepackage{url}            
\usepackage{booktabs}       
\usepackage{amsfonts}       
\usepackage{nicefrac}       
\usepackage{microtype}      
\usepackage{xcolor}         
\usepackage{float}
\usepackage{amsmath}
\usepackage{amsthm}
\usepackage{latexsym}
\usepackage{amssymb}
\usepackage{thmtools}
\usepackage{thm-restate}
\usepackage{etoolbox}
\usepackage{array}
\usepackage{booktabs}
\usepackage{makecell}
\usepackage{tabularx}
\usepackage{enumitem}
\usepackage{titlesec}
\usepackage{caption}
\usepackage{subcaption}
\usepackage{comment}
\usepackage{algorithm}
\usepackage{algorithmicx}
\usepackage[noend]{algpseudocode}
\usepackage{pifont}
\usepackage{verbatim,ifthen,alltt}
\usepackage{microtype}      
\usepackage{stmaryrd}
\usepackage{xfrac}
\usepackage{tikz}
\usepackage{pgf}

\usetikzlibrary{fit,calc,tikzmark}

\definecolor{darkslategray}{rgb}{0.18, 0.31, 0.31} 
\definecolor{platinum}{rgb}{0.9, 0.89, 0.89} 
\definecolor{gray}{rgb}{.4,.4,.4}
\definecolor{midgrey}{rgb}{0.5,0.5,0.5}
\definecolor{middarkgrey}{rgb}{0.35,0.35,0.35}
\definecolor{darkgrey}{rgb}{0.3,0.3,0.3}
\definecolor{darkred}{rgb}{0.7,0.1,0.1}
\definecolor{midblue}{rgb}{0.2,0.2,0.7}
\definecolor{darkblue}{rgb}{0.1,0.1,0.5}
\definecolor{darkgreen}{rgb}{0.1,0.5,0.1}
\definecolor{defseagreen}{cmyk}{0.69,0,0.50,0}

\def\pdfauthor{X. Huang and J. Marques-Silva}
\usepackage[pdftex%
,colorlinks=true%
,bookmarks=true%
,linkcolor=darkblue%
,citecolor=darkblue%
,urlcolor=darkblue%
,plainpages=false]{hyperref}
\hypersetup{%
pdfauthor={{\pdfauthor}},
pdftitle={{}}
}
\usepackage[noabbrev,nameinlink,capitalise]{cleveref}

\makeatletter
\def\thm@space@setup{\thm@preskip=5.0pt
\thm@postskip=0pt}
\makeatother

\newtheoremstyle{newstyle}      
{5.5pt} 
{-2.0pt} 
{\mdseries} 
{} 
{\bfseries} 
{.} 
{ } 
{} 
\theoremstyle{newstyle}

\newtheorem{proposition}{Proposition}

\newtheorem{example}{Example}
\newtheorem{remark}{Remark}

\crefname{theorem}{Theorem}{Theorems}
\crefname{lemma}{Lemma}{Lemmas}
\crefname{proposition}{Proposition}{Propositions}
\crefname{definition}{Definition}{Definitions}
\crefname{corollary}{Corollary}{Corollaries}
\crefname{example}{Example}{Examples}
\crefname{claim}{Claim}{Claims}
\crefname{assumption}{Assumption}{Assumptions}
\crefname{enumi}{}{}

\makeatletter
\renewenvironment{proof}[1][\proofname]{\par
  \pushQED{\qed}%
  \normalfont \topsep-2\p@\relax
  \trivlist
  \item[\hskip\labelsep\itshape
  #1\@addpunct{.}]\ignorespaces
}{%
  \popQED\endtrivlist\@endpefalse
}
\makeatother

\newcommand{\fml}[1]{{\mathcal{#1}}}

\newcommand{\tn}[1]{\textnormal{#1}}
\newcommand{\tbf}[1]{\textbf{#1}}
\newcommand{\mbf}[1]{\ensuremath\mathbf{#1}}
\newcommand{\msf}[1]{\ensuremath\mathsf{#1}}
\newcommand{\mbb}[1]{\ensuremath\mathbb{#1}}

\newcommand{\waxp}{\ensuremath\mathsf{WAXp}}
\newcommand{\wcxp}{\ensuremath\mathsf{WCXp}}
\newcommand{\axp}{\ensuremath\mathsf{AXp}}
\newcommand{\cxp}{\ensuremath\mathsf{CXp}}

\newcommand{\shap}{\ensuremath\msf{SHAP}}
\newcommand{\sv}{\ensuremath\msf{Sv}}

\newcommand{\oper}[1]{\ensuremath\mathsf{#1}}

\newcommand{\bigland}{\ensuremath\bigwedge}

\newcommand{\relevant}{\oper{Relevant}}
\newcommand{\irrelevant}{\oper{Irrelevant}}

\newcounter{tableeqn}[table]

\DeclareMathOperator*{\entails}{\vDash}

\DeclareMathOperator*{\limply}{\rightarrow}

\newcommand{\xnote}[1]{\medskip\noindent$\llbracket$\textcolor{darkred}{xiang}: \emph{\textcolor{middarkgrey}{#1}}$\rrbracket$\medskip}

\newcommand{\jnoteF}[1]{}

\newcolumntype{L}[1]{>{\raggedright\let\newline\\\arraybackslash\hspace{0pt}}m{#1}}
\newcolumntype{C}[1]{>{\centering\let\newline\\\arraybackslash\hspace{0pt}}m{#1}}
\newcolumntype{R}[1]{>{\raggedleft\let\newline\\\arraybackslash\hspace{0pt}}m{#1}}

\tikzset{
  0 my edge/.style={densely dashed, my edge},
  my edge/.style={-{Stealth[]}},
}

\def\HiLi{\leavevmode\rlap{\hbox to \linewidth{\color{platinum}\leaders\hrule height .8\baselineskip depth .5ex\hfill}}}

\titlespacing{\section}{0pt}{*2.15}{*1.0}
\titlespacing{\subsection}{0pt}{*1.25}{*0.75}
\titlespacing{\subsubsection}{0pt}{*0.35}{*0.5}
\titlespacing{\paragraph}{0pt}{*0.1}{*0.575}

\makeatletter
\newcommand\nparagraph{%
  \@startsection{paragraph}
    {4}
    {\z@}
    {0.225ex \@plus0.225ex \@minus.125ex}
    {-1em}
    {\normalfont\normalsize\bfseries}%
}
\makeatother

\setitemize{noitemsep,topsep=0pt,parsep=0pt,partopsep=0pt}
\setenumerate{noitemsep,topsep=0pt,parsep=0pt,partopsep=0pt}
\setlist{nosep,leftmargin=0.45cm}

\providetoggle{long}
\settoggle{long}{false}

\makeatletter
\algnewcommand{\LineComment}[1]{\Statex \hskip\ALG@thistlm \(\triangleright\) #1}
\makeatother

\title{%
  A Refutation of Shapley Values for Explainability
}

\author{%
  Xuanxiang Huang
  \\
  University of Toulouse \\
  Toulouse, France \\
  \texttt{xuanxiang.huang@univ-toulouse.fr} \\
  \And
  Joao Marques-Silva \\
  IRIT, CNRS, France \\
  Toulouse, France\\
  \texttt{joao.marques-silva@irit.fr} \\
}

\begin{document}

\maketitle

%
\begin{abstract}
  Recent work demonstrated the existence of Boolean functions for
  which Shapley values provide misleading information about the
  relative importance of features in rule-based explanations.
  Such misleading information was broadly categorized into a number of
  possible issues. Each of those issues relates with features being
  relevant or irrelevant for a prediction, and all are significant
  regarding the inadequacy of Shapley values for rule-based
  explainability.
  %
  %
  This earlier work devised a brute-force approach to identify Boolean
  functions, defined on small numbers of features, and also associated
  instances, which displayed such inadequacy-revealing issues, and so
  served as evidence to the inadequacy of Shapley values for
  rule-based explainability.
  However, an outstanding question is how frequently such
  inadequacy-revealing issues can occur for Boolean functions with
  arbitrary large numbers of features. It is plain that a brute-force
  approach would be unlikely to provide insights on how to tackle this
  question.
  %
  This paper answers the above question by proving that, for any
  number of features, there exist Boolean functions that exhibit one
  or more inadequacy-revealing issues, thereby contributing 
  decisive arguments against the use of Shapley values as the
  theoretical underpinning of feature-attribution methods in
  explainability.
\end{abstract}
%

%
\section{Introduction} \label{sec:intro}

Feature attribution is one of the most widely used approaches in
machine learning (ML) explainability, begin implemented with a variety
of different methods~\cite{kononenko-jmlr10,guestrin-kdd16,samek-bk19}.
Moreover, the use of Shapley values~\cite{shapley-ctg53} for feature
attribution ranks among the most popular
solutions~\cite{kononenko-jmlr10,kononenko-kis14,lundberg-nips17,lundberg-nips20,lundberg-naturemi20},
offering a widely accepted theoretical justification on how to assign
importance to features in machine learning (ML) model predictions.
Despite the success of using Shapley values for explainability, it is
also the case that their exact computation is in general
intractable~\cite{barcelo-aaai21,vandenbroeck-aaai21,vandenbroeck-jair22},
with tractability results for some families of boolean
circuits~\cite{barcelo-aaai21}.
As a result, a detailed assessment of the rigor of feature attribution
methods based on Shapley values, when compared with exactly computed
Shapley values has not been investigated.
Furthermore, the definition Shapley values (as well as its use in
explainability) is purely axiomatic, i.e.\ there exists \emph{no}
formal proof that Shapley values capture any specific properties
related with explainability (even if defining such properties might
prove elusive).

Feature selection represents a different alternative to feature
attribution.
The goal of feature selection is to select a set of features as
representing the reason for a prediction, i.e.\ if the selected
features take their assigned values, then the prediction cannot be
changed.
There are rigorous and non-rigorous approaches for selecting the
features that explain a prediction. This paper considers rigorous (or
model-precise) approaches for selecting such features.
Furthermore, it should be plain that feature selection must aim for
irredundancy, since otherwise it would suffice to report all features
as the explanation.
Given the universe of possible irreducible sets of feature selections
that explain a prediction, the features that do not occur in
\emph{any} such set are deemed \emph{irrelevant} for a prediction;
otherwise features that occur in one or more feature selections are
deemed \emph{relevant}.

Since both feature attribution and feature selection measure
contributions of features to explanations, one would expect that the
two approaches were related. However, this is not the case.
Recent work~\cite{hms-corr23} observed that feature attribution based
on Shapley values could produce \emph{misleading information} about
features, in that irrelevant features (for feature selection) could be
deemed more important (in terms of feature attribution) than relevant
features (also for feature selection).
Clearly, misleading information about the relative importance of
features can easily induce human decision makers in error, by
suggesting the \emph{wrong} features as those to analyze in greater
detail. Furthermore, situations where human decision makers can be
misled are inadmissible in high-risk or safety-critical uses of ML.
Furthermore, a number of possible misleading issues of Shapley values
for explainability were identified~\cite{hms-corr23}, and empirically
demonstrated to occur for some boolean functions.
%
The existence in practice of those misleading issues with Shapley
values for explainability is evidently problematic for their use as
the theoretical underpinning of feature attribution methods.

However, earlier work~\cite{hms-corr23} used a brute-force method to
identify boolean functions, defined on a very small number of
variables, where the misleading issues could be observed.
%
A limitation of this earlier work~\cite{hms-corr23} is that it offered 
no insights on how general the issues with Shapley values for
explainability are. For example, it could be the case that the
identified misleading issues might only occur for functions defined on
a very small number of variables, or in a negligible number of
functions, among the universe of functions defined on a given number
of variables.
If that were to be the case, then the issues with Shapley values for
explainability might not be that problematic. 

This paper proves that the identified misleading issues with Shapley
values for explainability are much more general that what was reported
in earlier work~\cite{hms-corr23}. Concretely, the paper proves that,
for any number of features larger than a small $k$ (either 2 or 3),
one can easily construct functions which exhibit the identified
misleading issues. 
The main implication of our results is clear: \emph{the use of Shapley
values for explainability can, for an arbitrary large number of
boolean (classification) functions, produce misleading information
about the relative importance of features}.

\paragraph{Organization.}
%
The paper is organized as follows.
\cref{sec:prelim} introduces the notation and definitions used
throughout the paper.
\cref{sec:cmp} revisits and extends the issues with Shapley values for
explainability reported in earlier work~\cite{hms-corr23}, and
illustrates the existence of those issues in a number of motivating
example boolean functions.
\cref{sec:negres} presents the paper's main results, proving that all
the issues with Shapley values for explainability reported in earlier
work~\cite{hms-corr23} occur for boolean functions with arbitrarily
larger number of variables.
(Due to lack of space, the detailed proofs are all included
in~\cref{sec:proofs}, and the paper includes only brief insights into
those proofs.)
Also, the proposed constructions offer ample confidence that the
number of functions displaying one or more of the issues is
significant.
\cref{sec:conc} concludes the paper.

\section{Preliminaries} \label{sec:prelim}

\paragraph{Boolean functions.}
%
Let $\mbb{B}=\{0,1\}$. The results in the paper consider boolean
functions, defined on $m$ boolean variables,
i.e.\ $\kappa:\mbb{B}^{m}\to\mbb{B}$.
(The fact that we consider only boolean functions does not restrict in
the significance of the results.)

In the rest of the paper, we will use the boolean functions shown
in~\cref{fig:runex}, which are represented by truth tables.
The highlighted rows will serve as concrete examples throughout.

\begin{figure*}
  \begin{minipage}{0.325\textwidth}
    \begin{subfigure}[t]{1\textwidth}
      \centering
      \renewcommand{\arraystretch}{0.95}
      \renewcommand{\tabcolsep}{0.5em}
      \scalebox{0.975}{
        \begin{tabular}{cccc} \toprule
          $x_1$ & $x_2$ & $x_3$ & $\kappa_{I1}(\mbf{x})$ \\ \toprule
          0 & 0 & 0 & 0 \\
          \tikzmarknode{a}{0} & 0 & 1 & \tikzmarknode{b}{0} \\
          0 & 1 & 0 & 0 \\
          0 & 1 & 1 & 0 \\
          1 & 0 & 0 & 0 \\
          1 & 0 & 1 & 1 \\
          1 & 1 & 0 & 1 \\
          1 & 1 & 1 & 1 \\
          \bottomrule
        \end{tabular}
      }
      \caption{Function $\kappa_{I1}$} \label{ex:k1}
      \begin{tikzpicture}[overlay,remember picture]
        \node[draw=midblue, thin, xshift=-0.35pt, yshift=-0.5pt, inner sep=1.5pt, fit=(a) (b)] {};
      \end{tikzpicture}
    \end{subfigure}

    \begin{subfigure}[b]{1\textwidth}
      \centering
      \renewcommand{\arraystretch}{0.95}
      \renewcommand{\tabcolsep}{0.5em}
      \scalebox{0.95}{
        \begin{tabular}{cccc} \toprule
          $x_1$ & $x_2$ & $x_3$ & $\kappa_{I3}(\mbf{x})$ \\ \toprule
          0 & 0 & 0 & 0\\
          0 & 0 & 1 & 0\\
          0 & 1 & 0 & 0\\
          0 & 1 & 1 & 1\\
          1 & 0 & 0 & 1\\
          1 & 0 & 1 & 0\\
          1 & 1 & 0 & 1\\
          \tikzmarknode{c}{1} & 1 & 1 & \tikzmarknode{d}{1}\\
          \bottomrule
        \end{tabular}
      }
      \caption{Function $\kappa_{I3}$} \label{ex:k3}
      \begin{tikzpicture}[overlay,remember picture]
        \node[draw=midblue, thin, xshift=-0.75pt, yshift=1.75pt, inner sep=1.5pt, fit=(c) (d)] {};
      \end{tikzpicture}
    \end{subfigure}
  \end{minipage}
  \begin{minipage}{0.325\textwidth}
    \begin{subfigure}{1\textwidth}
      \centering
      \renewcommand{\arraystretch}{0.95}
      \renewcommand{\tabcolsep}{0.5em}
      \scalebox{0.95}{
        \begin{tabular}[t]{ccccc} \toprule
          $x_1$ & $x_2$ & $x_3$ & $x_4$ & $\kappa_{I4}(\mbf{x})$ \\ \toprule
          0 & 0 & 0 & 0 & 0\\
          0 & 0 & 0 & 1 & 0\\
          0 & 0 & 1 & 0 & 0\\
          \tikzmarknode{e}{0} & 0 & 1 & 1 & \tikzmarknode{f}{0}\\
          0 & 1 & 0 & 0 & 0\\
          0 & 1 & 0 & 1 & 0\\
          0 & 1 & 1 & 0 & 0\\
          0 & 1 & 1 & 1 & 1\\
          1 & 0 & 0 & 0 & 0\\
          1 & 0 & 0 & 1 & 0\\
          1 & 0 & 1 & 0 & 1\\
          1 & 0 & 1 & 1 & 0\\
          1 & 1 & 0 & 0 & 1\\
          1 & 1 & 0 & 1 & 1\\
          1 & 1 & 1 & 0 & 1\\
          1 & 1 & 1 & 1 & 1\\
          \bottomrule
        \end{tabular}
      }
      \caption{Function $\kappa_{I4}$} \label{ex:k4}
      \begin{tikzpicture}[overlay,remember picture]
        \node[draw=midblue, thin, xshift=-1.5pt, yshift=2.5pt, inner sep=1.5pt, fit=(e) (f)] {};
      \end{tikzpicture}
    \end{subfigure}
  \end{minipage}
  \begin{minipage}{0.325\textwidth}
    \begin{subfigure}[t]{1\textwidth}
      \centering
      \renewcommand{\arraystretch}{0.95}
      \renewcommand{\tabcolsep}{0.5em}
      \scalebox{0.95}{
        \begin{tabular}{ccccc} \toprule
          $x_1$ & $x_2$ & $x_3$ & $x_4$ & $\kappa_{I5}(\mbf{x})$ \\ \toprule
          0 & 0 & 0 & 0 & 0\\
          0 & 0 & 0 & 1 & 0\\
          0 & 0 & 1 & 0 & 0\\
          0 & 0 & 1 & 1 & 0\\
          0 & 1 & 0 & 0 & 0\\
          0 & 1 & 0 & 1 & 0\\
          0 & 1 & 1 & 0 & 0\\
          0 & 1 & 1 & 1 & 1\\
          1 & 0 & 0 & 0 & 0\\
          1 & 0 & 0 & 1 & 0\\
          1 & 0 & 1 & 0 & 0\\
          1 & 0 & 1 & 1 & 1\\
          1 & 1 & 0 & 0 & 0\\
          1 & 1 & 0 & 1 & 1\\
          1 & 1 & 1 & 0 & 0\\
          \tikzmarknode{g}{1} & 1 & 1 & 1 & \tikzmarknode{h}{0}\\
          \bottomrule
        \end{tabular}
      }
      \caption{Function $\kappa_{I5}$} \label{ex:k5}
      \begin{tikzpicture}[overlay,remember picture]
        \node[draw=midblue, thin, xshift=-1.25pt, yshift=3.75pt, inner sep=1.5pt, fit=(g) (h)] {};
      \end{tikzpicture}
    \end{subfigure}
  \end{minipage}
  \caption{Example functions for issues \cref{en:i1}, \cref{en:i3},
    \cref{en:i4}, \cref{en:i5}, respectively $\kappa_{I1}$,
    $\kappa_{I3}$, $\kappa_{I4}$, $\kappa_{I5}$} \label{fig:runex}
\end{figure*}

\paragraph{Classification in ML.}
%
A classification problem is defined on a set of features
$\fml{F}=\{1,\ldots,m\}$, each with domain $\mbb{D}_i$, and a set of
classes $\fml{K}=\{c_1,c_2,\ldots,c_K\}$.
(As noted above, we will assume $\mbb{D}_i=\mbb{B}$ for
$1\le{i}\le{m}$, but domains could be categorical or ordinal. Also, we
will assume $\fml{K}=\mbb{B}$.)
Feature space $\mbb{F}$ is defined as the cartesian product of the
domains of the features, in order:
$\mbb{F}=\mbb{D}_1\times\cdots\times\mbb{D}_m$, which will be
$\mbb{B}^m$ throughout the paper.
A classification function is a non-constant map from feature space
into the set of classes, $\kappa:\mbb{F}\to\fml{K}$.
(Clearly, a classifier would be useless if the classification function
were constant.)
Throughout the paper, we will not distinguish between classifiers and
boolean functions.
An instance is a pair $(\mbf{v},c)$ representing a point $\mbf{v}=(v_1,\dots,v_m)$ in
feature space, and the classifier's prediction,
i.e.\ $\kappa(\mbf{v})=c$.
Moreover, we let $\mbf{x}=(x_1,\dots,x_m)$ denote an arbitrary point in the feature space.
Abusing notation, we will also use $\mbf{x}_{a..b}$ to denote $x_a,\ldots,x_b$,
and $\mbf{v}_{a..b}$ to denote $v_a,\ldots,v_b$.
Finally, a classifier $\fml{M}$ is a tuple
$(\fml{F},\mbb{F},\fml{K},\kappa)$.
In addition, an explanation problem $\fml{E}$ is a tuple
$(\fml{M},(\mbf{v},c))$, where
$\fml{M}=(\fml{F},\mbb{F},\fml{K},\kappa)$ is a classifier.

\paragraph{Shapley values for explainability.}
%
Shapley values were first introduced by
L.~Shapley~\cite{shapley-ctg53} in the context of game theory.
Shapley values have been extensively used for explaining the
predictions of ML models,
e.g.~\cite{kononenko-jmlr10,kononenko-kis14,zick-sp16,lundberg-nips17,jordan-iclr19,taly-cdmake20,lakkaraju-nips21,watson-facct22},
among a vast number of recent examples.
The complexity of computing Shapley values (as proposed in
SHAP~\cite{lundberg-nips17}) has been studied in recent
years~\cite{barcelo-aaai21,vandenbroeck-aaai21,barcelo-corr21,vandenbroeck-jair22}.
This section provides a brief overview of Shapley values. 
Throughout the section, we adapt the notation used in recent
work~\cite{barcelo-aaai21,barcelo-corr21}, which builds on the work
of~\cite{lundberg-nips17}.

Let $\Upsilon:2^{\fml{F}}\to2^{\mbb{F}}$ be defined by%
\footnote{%
When defining concepts, we will show the necessary parameterizations.
However, in later uses, those parameterizations will be omitted, for
simplicity.},
\begin{equation} \label{eq:upsilon}
  \Upsilon(\fml{S};\mbf{v})=\{\mbf{x}\in\mbb{F}\,|\,\land_{i\in\fml{S}}x_i=v_i\}
\end{equation}
i.e.\ for a given set $\fml{S}$ of features, and parameterized by
the point $\mbf{v}$ in feature space, $\Upsilon(\fml{S};\mbf{v})$
denotes all the points in feature space that have in common with
$\mbf{v}$ the values of the features specified by $\fml{S}$. 

Also, let $\phi:2^{\fml{F}}\to\mbb{R}$ be defined by,
\begin{equation} \label{eq:phi}
  \phi(\fml{S};\fml{M},\mbf{v})=\frac{1}{2^{|\fml{F}\setminus\fml{S}|}}\sum_{\mbf{x}\in\Upsilon(\fml{S};\mbf{v})}\kappa(\mbf{x})
\end{equation}
For the purposes of this paper, we consider solely a uniform input
distribution, and so the dependency on the input distribution is not
accounted for. A more general formulation is considered in related
work~\cite{barcelo-aaai21,barcelo-corr21}. However, assuming a uniform
distribution suffices for the purposes of this paper.
As a result, 
given a set $\fml{S}$ of features, $\phi(\fml{S};\fml{M},\mbf{v})$
represents the average value of the classifier over the points of
feature space represented by $\Upsilon(\fml{S};\mbf{v})$.
%

Finally, let $\sv:\fml{F}\to\mbb{R}$ be defined by\footnote{%
We distinguish $\shap(\cdot;\cdot,\cdot)$ from $\sv(\cdot;\cdot,\cdot)$. Whereas
$\shap(\cdot;\cdot,\cdot)$ represents the value computed by the tool
SHAP~\cite{lundberg-nips17}, $\sv(\cdot;\cdot,\cdot)$ represents the Shapley
value in the context of (feature attribution based) explainability, as
studied in a number of
works~\cite{kononenko-jmlr10,kononenko-kis14,lundberg-nips17,barcelo-aaai21,vandenbroeck-aaai21,vandenbroeck-jair22}. Thus, 
$\shap(\cdot;\cdot,\cdot)$ is a heuristic approximation of $\sv(\cdot;\cdot,\cdot)$.},
\begin{equation} \label{eq:sv}
  \sv(i;\fml{M},\mbf{v})=\sum_{\fml{S}\subseteq(\fml{F}\setminus\{i\})}\frac{|\fml{S}|!(|\fml{F}|-|\fml{S}|-1)!}{|\fml{F}|!}\left(\phi(\fml{S}\cup\{i\};\fml{M},\mbf{v})-\phi(\fml{S};\fml{M},\mbf{v})\right)
\end{equation}
Given an instance $(\mbf{v},c)$, the Shapley value assigned to each
feature measures the \emph{contribution} of that feature with respect
to the prediction. 
A positive/negative value indicates that the feature can contribute to
changing the prediction, whereas a value of 0 indicates no
contribution.
%

\jnoteF{To
  cite:~\cite{shapley-ctg53,barcelo-aaai21,vandenbroeck-aaai21,vandenbroeck-jair22}.}

\begin{example} \label{ex:calcsv}
  We consider the example boolean functions of~\cref{fig:runex}.
  If the functions are represented by a truth table, then the Shapley
  values can be computed in polynomial time on the size of the truth
  table, since the number of subsets considered in~\eqref{eq:sv} is
  also polynomial on the size of the truth table~\cite{hms-corr23}. 
  (Observe that for each subset used in~\eqref{eq:sv}, we can
  use the truth table for computing the average values
  in~\eqref{eq:phi}.)
  For example, for $\kappa_{I1}$ and for the point in feature space
  $(0,0,1)$, one can compute the following Shapley values:
  $\sv(1)=-0.417$,  $\sv(2)=-0.042$, and $\sv(3)=0.083$.
\end{example}

\paragraph{Logic-based explanations.}
%
There has been recent work on developing formal definitions of
explanations.
One type of explanations are \emph{abductive
explanations}~\cite{inms-aaai19} (AXp), which corresponds to a
PI-explanations~\cite{darwiche-ijcai18} in the case of boolean  
classifiers. AXp's represent prime implicants of the discrete-valued
classifier function (which computes the predicted class).
AXp's can also be viewed as an instantiation of logic-based
abduction~\cite{gottlob-ese90,selman-aaai90,bylander-aij91,gottlob-jacm95}.
Throughout this paper we will opt to use the acronym AXp to refer to
abductive explanations.

Let us consider a given classifier, computing a classification function
$\kappa$ on feature space $\mbb{F}$, a point $\mbf{v}\in\mbb{F}$, with
prediction $c=\kappa(\mbf{v})$, and let $\fml{X}$ denote a subset of
the set of features $\fml{F}$, 
$\fml{X}\subseteq\fml{F}$. $\fml{X}$ is a weak AXp for the instance
$(\mbf{v},c)$ if,
\begin{equation} \label{eq:axp1}
  \begin{array}{rcl}
    \waxp(\fml{X};\fml{M},\mbf{v}) & ~:=~~ &
    \forall(\mbf{x}\in\mbb{F}).%
    \left[\bigwedge_{i\in\fml{X}}(x_i=v_i)\right]\limply(\kappa(\mbf{x})=c)\\ 
  \end{array}
\end{equation}
where $c=\kappa(\mbf{v})$.
%
%
Thus, given an instance $(\mbf{v},c)$, a (weak) AXp is a subset of
features which, if fixed to the values dictated by $\mbf{v}$, then the
prediction is guaranteed to be $c$, independently of the values
assigned to the other features.
Moreover, $\fml{X}\subseteq\fml{F}$ is an AXp if, besides being a weak
AXp, it is also subset-minimal, i.e.
\begin{equation} \label{eq:axp2a}
  \begin{array}{rcl}
    \axp(\fml{X};\fml{M},\mbf{v}) & ~:=~~ &
    \waxp(\fml{X};\fml{M},\mbf{v})\land\forall(\fml{X}'\subsetneq\fml{X}).\neg\waxp(\fml{X}';\fml{M},\mbf{v})\\
  \end{array}
\end{equation}  
Observe that an AXp can be viewed as a possible irreducible answer to
a ``\tbf{Why?}'' question, i.e.\ why is the classifier's prediction
$c$?
It should be plain in this work, but also in earlier work, that the
representation of AXp's using subsets of features aims at simplicity.
The sufficient condition for the prediction is evidently the
conjunction of literals associated with the features contained in the
AXp.

\begin{example} \label{ex:calca}
  Similar to the computation of Shapley values, given a truth table
  representation of a function, and for a given instance, there is a
  polynomial-time algorithm for computing the
  AXp's~\cite{hms-corr23}. 
  For example, for function $\kappa_{I4}$ (see~\cref{ex:k4}), and
  for the instance $((0,0,1,1),0)$, it can be observed that, if
  features 3 and 4 are allowed to take other values, the prediction
  remains at 0. Hence, $\{1,2\}$ is an WAXp, which is easy to conclude
  that it is also an AXp.
  When interpreted as a rule, the AXp would yield the rule:
  \[
  \tn{IF} \quad \neg{x_1}\land\neg{x_2} \quad \tn{THEN} \quad
  \kappa(\mbf{x})=0
  \]
  In a similar way, if features 1 and 3 are allowed to
  take other values, the prediction remains at 0. Hence, $\{2,4\}$ is
  another WAXp (which can easily be shown to be an AXp).
  Furthermore, considering all other possible subsets of fixed
  features, allows us to conclude that there are no more AXp's.
\end{example}

Similarly to the case of AXp's, one can define (weak) contrastive
explanations (CXp's)~\cite{miller-aij19,inams-aiia20}.
$\fml{Y}\subseteq\fml{F}$ is a weak CXp for the instance $(\mbf{v},c)$
if,
\begin{equation} \label{eq:cxp1}
  \begin{array}{rcl}
    \wcxp(\fml{Y};\fml{M},\mbf{v}) & ~:=~~ & \exists(\mbf{x}\in\mbb{F}).%
    \left[\bigwedge_{i\not\in\fml{Y}}(x_i=v_i)\right]\land(\kappa(\mbf{x})\not=c)\\ 
  \end{array}
\end{equation}
(As before, for simplicity we keep the parameterization of $\wcxp$ on
$\kappa$, $\mbf{v}$ and $c$ implicit.)
Thus, given an instance $(\mbf{v},c)$, a (weak) CXp is a subset of
features which, if allowed to take any value from their domain, then
there is an assignment to the features that changes the prediction to
a class other than $c$, this while the features not in the explanation
are kept to their values. 

Furthermore, a set $\fml{Y}\subseteq\fml{F}$ is a CXp if, besides
being a weak CXp, it is also subset-minimal, i.e.
\begin{equation} \label{eq:cxp2a}
  \begin{array}{rcl}
    \cxp(\fml{Y};\fml{M},\mbf{v}) & ~:=~~ &
    \wcxp(\fml{Y};\fml{M},\mbf{v})\land\forall(\fml{Y}'\subsetneq\fml{Y}).\neg\wcxp(\fml{Y}';\fml{M},\mbf{v})\\
  \end{array}
\end{equation}  
A CXp can be viewed as a possible irreducible answer to a ``\tbf{Why
  Not?}'' question, i.e.\ why isn't the classifier's prediction a
class other than $c$?

\begin{example} \label{calcc}
  For the example function $\kappa_{I4}$ (see~\cref{ex:k4}), and instance
  $((0,0,1,1),0)$, if we fix features 1, 3 and 4, respectively to 0, 1
  1, then by allowing feature 2 to change value, we see that the
  prediction changes, e.g.\ by considering the point $(0,1,1,1)$ with
  prediction 1. Thus, $\{2\}$ is a CXp.
  In a similar way, by fixing the features 2 and 3, respectively to 0
  and 1, then by allowing features 1 and 4 to change value, we
  conclude that the prediction changes. Hence, $\{1,4\}$ is also a
  CXp.
\end{example}


%
%
%

The sets of AXp's and CXp's are defined as follows:
\begin{equation}
  \begin{array}{l}
    \mbb{A}(\fml{E})=\{\fml{X}\subseteq\fml{F}\,|\,\axp(\fml{X};\fml{M},\mbf{v})\}\\[3pt]
    \mbb{C}(\fml{E})=\{\fml{Y}\subseteq\fml{F}\,|\,\cxp(\fml{Y};\fml{M},\mbf{v})\}
  \end{array}
\end{equation}
(The parameterization on $\fml{M}$ and $\mbf{v}$ is unnecessary, since
the explanation problem $\fml{E}$ already accounts for those.)
Moreover, let
$F_{\mbb{A}}(\fml{E})=\cup_{\fml{X}\in\mbb{A}(\fml{E})}\fml{X}$ and
$F_{\mbb{C}}(\fml{E})=\cup_{\fml{Y}\in\mbb{C}(\fml{E})}\fml{Y}$.
$F_{\mbb{A}}(\fml{E})$ aggregates the features occurring in any
abductive explanation, whereas $F_{\mbb{C}}(\fml{E})$ aggregates the
features occurring in any contrastive explanation.
In addition, minimal hitting set duality between AXp's and 
CXp's~\cite{inams-aiia20} yields the following result\footnote{All
proofs are included in~\cref{sec:proofs}.}.
\begin{restatable}{proposition}{PropDualTwo}%
  ${F}_{\mbb{A}}(\fml{E})={F}_{\mbb{C}}(\fml{E})$.
  \label{prop:dual2}
\end{restatable}

\paragraph{Feature (ir)relevancy in explainability.}
%
Given the definitions above, we have the following characterization of
features~\cite{hiims-kr21,hims-aaai23,hcmpms-tacas23}:
\begin{enumerate}[nosep]
\item A feature $i\in\fml{F}$ is \emph{necessary} if
  $\forall(\fml{X}\in\mbb{A}(\fml{E})).i\in\fml{X}$.
\item A feature $i\in\fml{F}$ is \emph{relevant} if
  $\exists(\fml{X}\in\mbb{A}(\fml{E})).i\in\fml{X}$.
\item A feature is \emph{irrelevant} if it is not relevant, i.e.\
  $\forall(\fml{X}\in\mbb{A}(\fml{E})).i\not\in\fml{X}$.
\end{enumerate}
By~\cref{prop:dual2}, the definitions of necessary and relevant
feature could instead use $\mbb{C}(\fml{E})$.
Throughout the paper, we will use the predicate $\irrelevant(i)$ which 
holds true if feature $i$ is irrelevant, and predicate $\relevant(i)$
which holds true if feature $i$ is relevant.
Furthermore, it should be noted that feature irrelevancy is a fairly
demanding condition in that, a feature $i$ is irrelevant if it is not
included in \emph{any} subset-minimal set of features that is
sufficient for the prediction.

\begin{example} \label{ex:calcr}
  For the example function $\kappa_{I4}$ (see~\cref{ex:k4}), and
  from~\cref{ex:calca}, and instance $((0,0,1,1),0)$, it becomes
  clear that feature 3 is irrelevant. Similarly, it is easy to
  conclude that features 1, 2 and 4 are relevant.
\end{example}

\jnoteF{To cite:~\cite{hiims-kr21,hims-aaai23}}

\paragraph{How irrelevant are irrelevant features?}
The fact that a feature is declared irrelevant for an explanation
problem $\fml{E}=(\fml{M},(\mbf{v},c))$ is significant.
Given the minimal hitting set duality between abductive and
contrastive explanations, then an irrelevant features does not occur 
neither in any abductive explanation, nor in any contrastive
explanation.
Furthermore, from the definition of AXp, each abductive explanation
for $\fml{E}$ can be represented as a logic rule. Let $\fml{R}$ denote
the set of \emph{all irreducible} logic rules which can be used to
predict $c$, given the literals dictated by $\mbf{v}$. Then, an 
irrelevant feature does not occur in \emph{any} of those rules.
\cref{ex:calcr} illustrates the irrelevancy of feature 3, in that
feature 3 would not occur in \emph{any} irreducible rule for
$\kappa_{I4}$ when predicting $0$ using literals consistent with
$(0,0,1,1)$.

To further strengthen the above discussion, let us consider a (feature
selection based) explanation $\fml{X}\subseteq\fml{F}$ such that
$\waxp(\fml{X})$ holds (i.e.\ \eqref{eq:axp1} is true, and so
$\fml{X}$ is sufficient for the prediction). Moreover, let
$i\in\fml{F}$ be an irrelevant feature, such that $i\in\fml{X}$. Then,
by definition of irrelevant feature, there \emph{must} exist
some $\fml{Z}\subseteq(\fml{X}\setminus\{i\})$, such that
$\waxp(\fml{Z})$ also holds (i.e.\ $\fml{Z}$ is \emph{also} sufficient
for the prediction).
It is simple to understand why such set $\fml{Z}$ must exist.
By definition of irrelevant feature, and because $i\in\fml{X}$, then
$\fml{X}$ is not an AXp. However, there must exist an AXp
$\fml{W}\subsetneq\fml{X}$ which, by definition of irrelevant feature,
must not include $i$.
Furthermore, and invoking Occam's razor\footnote{%
Here, we adopt a fairly standard definition of Occam's
razor~\cite{haussler-ipl87}: \emph{given two explanations of the data,
all other things being equal, the simpler explanation is
preferable}.}, there is no reason to select $\fml{X}$ over $\fml{Z}$,
and this remark applies to \emph{any} set of features containing some
irrelevant feature.
%


\paragraph{Related work.}
Shapley values for explainability is one of the hallmarks of feature
attribution methods in
XAI~\cite{kononenko-jmlr10,kononenko-kis14,zick-sp16,lundberg-nips17,jordan-iclr19,lundberg-naturemi20,taly-cdmake20,lundberg-nips20,feige-nips20,covert-aistats21,feige-iclr21,lakkaraju-nips21,covert-iclr22,giannotti-ccai22,watson-facct22,magazzeni-facct22,giannotti-eg22,giannotti-dmkd22,xie-bigdata22,giannotti-epjds22}.
Motivated by the success of Shapley values for explainability, there
exists a burgeoning body of work on using Shapley values for
explainability (e.g.~\cite{jansen-dphm20,yu-tc20,withnell-bb21,inoguchi-sr21,moncada-naturesr21,baptista-aij22,alsinglawi-sr22,zhang-fo22,ladbury-go22,alabi-ijmi22,sorayaie-midm22,zarinshenas-ro22,ma-er22,wang-er22,liu-bbe22,acharya-cmpb22,lund-diagnostics22,menegaz-ieee-sp22,menegaz-ieee-jbhi23,huang-plosone23,adeoye-oo23}).
Recent work studied the complexity of exactly computing Shapley values
in the context of
explainability~\cite{barcelo-aaai21,vandenbroeck-aaai21,vandenbroeck-jair22}.
Finally, there have been proposals for the exact computation of
Shapley values in the case of circuit-based
classifiers~\cite{barcelo-aaai21}.
Although there exist some differences in the proposals for the use of
Shapley values for explainability, the basic formulation is the same
and can be expressed as in~\cref{sec:prelim}.

A number of authors have reported pitfalls with the use
of SHAP and Shapley values as a measure of feature
importance~\cite{shrapnel-corr19,friedler-icml20,najmi-icml20,taly-cdmake20,nguyen-ieee-access21,procaccia-aaai21,sharma-aies21,guigue-icml21,taly-uai21,friedler-nips21,roder-mlwa22}.
However, these earlier works do not identify fundamental flaws with
the use of Shapley values in explainability.
Attempts at addressing those pitfalls include proposals to integrate
Shapley values with abductive explanations, as reported in
recent work~\cite{labreuche-sum22}. 

Formal explainability is a fairly recent topic of research. Recent
accounts
include~\cite{msi-aaai22,marquis-dke22,ms-corr22,darwiche-jlli23}.

Recent work~\cite{hms-corr23} argued for the inadequacy of Shapley
values for explainability, by demonstrating experimentally that the
information provided by Shapley values can be misleading for a human 
decision-maker.
The approach proposed in~\cite{hms-corr23} is based on exhaustive
function enumeration, and so does not scale beyond a few features.
However, this paper uses the truth-table algorithms outlined
in~\cite{hms-corr23}, in all the examples, both for computing Shapley
values, for computing explanations, and for deciding feature
relevancy.

\section{Relating Shapley Values with Feature Relevancy} \label{sec:cmp}

Recent work~\cite{hms-corr23} showed the existence of boolean
functions (with up to four variables) that revealed a number of issues 
with Shapley values for explainability. All those issues are related
with taking feature relevancy into consideration.
(In~\cite{hms-corr23}, these functions were searched by exhaustive
enumeration of all the boolean functions up to a threshold on the
number of variables.) 
%

\paragraph{Issues with Shapley values for explainability.}
In this paper, we consider the following main issues of Shapley values
for explainability:
\begin{enumerate}[nosep,topsep=1.5pt,itemsep=1.5pt,label=\textbf{I\arabic*.},ref=\small\textrm{I\arabic*},leftmargin=0.75cm]  
\item For a boolean classifier, with an instance $(\mbf{v},c)$, and
  feature $i$ such that, \label{en:i1}
  \[
  \msf{Irrelevant}(i)\land\left(\msf{Sv}(i)\not=0\right)
  \]
  Thus, an~\cref{en:i1} issue is such that the feature is irrelevant,
  but its Shapley value is non-zero.
\item For a boolean classifier, with an instance $(\mbf{v},c)$ and
  features $i_1$ and $i_2$ such that, \label{en:i2}
  \[
  \begin{array}{l}
    \msf{Irrelevant}(i_1)\land\msf{Relevant}(i_2)\land
    \left(|\msf{Sv}(i_1)|>|\msf{Sv}(i_2)|\right)
  \end{array}
  \]
  Thus, an~\cref{en:i2} issue is such that there is at least one
  irrelevant feature exhibiting a Shapley value larger (in absolute
  value) than the Shapley of a relevant feature.
\item For a boolean classifier, with instance $(\mbf{v},c)$,
  and feature $i$ such that, \label{en:i3}
  \[
  \msf{Relevant}(i)\land\left(\msf{Sv}(i)=0\right)
  \]
  Thus, an~\cref{en:i3} issue is such that the feature is relevant,
  but its Shapley value is zero.
\item For a boolean classifier, with instance $(\mbf{v},c)$, and
  features $i_1$ and $i_2$ such that, \label{en:i4}
  \[
  [\msf{Irrelevant}(i_1)\land\left(\msf{Sv}(i_1)\not=0\right)]
  \land
  [\msf{Relevant}(i_2)\land\left(\msf{Sv}(i_2)=0\right)]
  \]
  Thus, an~\cref{en:i4} issue is such that there is at least one
  irrelevant feature with a non-zero Shapley value and a relevant
  feature with a Shapley value of 0.
\item For a boolean classifier, with instance $(\mbf{v},c)$ and
  feature $i$ such that, \label{en:i5}
  \[
    [\msf{Irrelevant}(i)\land
      \forall_{1\le{j}\le{m},j\not=i}\left(|\msf{Sv}(j)|<|\msf{Sv}(i)|\right)]
  \]
  Thus, an~\cref{en:i5} issue is such that there is one irrelevant
  feature exhibiting the highest Shapley value (in absolute value).
  (\cref{en:i5} can be viewed as a special case of the other issues,
  and so it is not analyzed separately in earlier
  work~\cite{hms-corr23}.)
\end{enumerate}
The issues above are all related with Shapley values for
explainability giving \emph{misleading information} to a human
decision maker, by assigning some importance to irrelevant features,
by not assigning enough importance to relevant features, by assigning
more importance to irrelevant features than to relevant features and,
finally, by assigning the most importance to irrelevant features. 

In the rest of the paper we consider mostly~\cref{en:i1},
\cref{en:i3}, \cref{en:i4} and \cref{en:i5}, given that \cref{en:i5}
implies \cref{en:i2}.

\begin{restatable}{proposition}{PropIfiveToItwo}
  If a classifier and instance exhibits issue~\cref{en:i5}, then they
  also exhibit issue~\cref{en:i2}.
\end{restatable}


\paragraph{Examples.}
This section studies the example functions of~\cref{fig:runex}, which
were derived from the main results of this paper
(see~\cref{sec:negres}). These example functions will then be used to
motivate the rationale for how those results are proved.
In all cases, the reported Shapley values are computed using the
truth-table algorithm outlined in earlier work~\cite{hms-corr23}.
Similarly, the relevancy/irrelevancy claims of features use the
truth-table algorithms outlined in earlier work~\cite{hms-corr23}.

\begin{table}[t]
  \centering
  \caption{Examples of issues of Shapley values for functions in~\cref{fig:runex}} \label{tab:runex}
  \scalebox{0.95}{\renewcommand{\tabcolsep}{0.35em}
\begin{tabular}{ccccC{2.75cm}c} \toprule
  Case & Instance & Relevant & Irrelevant & $\sv$'s & Justification
  \\ \toprule
  \cref{en:i1} &
  $((0,0,1),0)$ &
  $1$ &
  $2,3$ &
  $\begin{array}{l}\sv(1)=-0.417\\\sv(2)=-0.042\\\sv(3)=0.083\\\end{array}$ &
  $\irrelevant(3)\land\sv(3)\not=0$
  \\ \midrule
  \cref{en:i3} &
  $((1,1,1),1)$ &
  $1,2,3$ &
  $\tn{--}$ &
  $\begin{array}{l}\sv(1)=0.125\\\sv(2)=0.375\\\sv(3)=0.000\\\end{array}$ &
  $\relevant(3)\land\sv(3)=0$
  \\ \midrule
  \cref{en:i4} &
  $((0,0,1,1),0)$ &
  $1,2,4$ &
  $3$ &
  $\begin{array}{l}\sv(1)=-0.125\\\sv(2)=-0.333\\\sv(3)=0.083\\\sv(4)=0.000\\\end{array}$ &
  $\begin{array}{l}\irrelevant(3)\land\sv(3)\not=0\land\\\relevant(4)\land\sv(4)=0\end{array}$
  \\ \midrule
  \cref{en:i5} &
  $((1,1,1,1),0)$ &
  $1,2,3$ &
  $4$ &
  $\begin{array}{l}\sv(1)=-0.12\\\sv(2)=-0.12\\\sv(3)=-0.12\\\sv(4)=0.17\\\end{array}$ &
  $\begin{array}{l}\irrelevant(4)\land\\\forall(j\in\{1,2,3\}).|\sv(j)|<\sv(4)|\end{array}$
  \\ 
  
  \bottomrule
\end{tabular}
}
\end{table}

\begin{example} \label{ex:k1ex}
  \cref{ex:k1} illustrates a boolean function that exhibits
  issue~\cref{en:i1}.
  By inspection, we can conclude that the function shown corresponds
  to
  $\kappa_{I1}(x_1,x_2,x_3)=(x_1 \land x_2 \land \neg{x}_3) \lor (x_1 \land x_3)$.
  Moreover, for the instance $((0,0,1),0)$,~\cref{tab:runex}
  confirms that an issue~\cref{en:i1} is identified.
\end{example}

\begin{example} \label{ex:k3ex}
  \cref{ex:k3} illustrates a boolean function that exhibits
  issue~\cref{en:i3}.
  By inspection, we can conclude that the function shown corresponds
  to
  $\kappa_{I3}(x_1,x_2,x_3)=(x_1 \land \neg{x}_3) \lor (x_2 \land x_3)$. 
  Moreover, for the instance $((1,1,1),1)$,~\cref{tab:runex}
  confirms that an issue~\cref{en:i3} is identified.
\end{example}

\begin{example} \label{ex:k4ex}
  \cref{ex:k4} illustrates a boolean function that exhibits
  issue~\cref{en:i4}.
  By inspection, we can conclude that the function shown corresponds
  to
  $\kappa_{I4}(x_1,x_2,x_3,x_4)=(x_1 \land x_2 \land \neg{x}_3) \lor (x_1 \land x_3 \land \neg{x}_4) \lor (x_2 \land x_3 \land x_4)$.
  Moreover, for the instance $((0,0,1,1),0)$,~\cref{tab:runex}
  confirms that an issue~\cref{en:i4} is identified.
\end{example}

\begin{example} \label{ex:k5ex}
  \cref{ex:k5} illustrates a boolean function that exhibits
  issue~\cref{en:i5}.
  By inspection, we can conclude that the function shown corresponds
  to
  $\kappa_{I5}(x_1,x_2,x_3,x_4)=((x_1 \land x_2 \land \neg{x}_3) \lor (x_1 \land x_3 \land \neg{x}_2) \lor (x_2 \land x_3 \land \neg{x}_1))\land x_4 $.
  Moreover, for the instance $((1,1,1,1),0)$,~\cref{tab:runex}
  confirms that an issue~\cref{en:i5} is identified.
\end{example}

\jnoteF{Build example functions using the constructions in the proofs,
  and the algorithm in~\cite{barcelo-aaai21} to illustrate the
  problem(s).\\
  Alternatively, given the truth-table, implement script for computing Sv's.
}

It should be underscored that Shapley values for explainability are
\emph{not} expected to give misleading information. 
Indeed, it is widely accepted that Shapley values measure the actual
\emph{influence} of a
feature~\cite{kononenko-jmlr10,kononenko-kis14,lundberg-nips17,barcelo-aaai21,vandenbroeck-aaai21}.
Concretely,~\cite{kononenko-jmlr10} reads: ``\emph{...if a feature
has no influence on the prediction it is assigned a contribution of
0.}''
But~\cite{kononenko-jmlr10} also reads ``\emph{According to the 2nd
axiom, if two features values have an identical influence on the 
prediction they are assigned contributions of equal size. The 3rd 
axiom says that if a feature has no influence on the prediction it is
assigned a contribution of 0.}'' (In this last quote, the axioms refer
to the axiomatic characterization of Shapley values.)
Furthermore, one might be tempted to look at the value of the
prediction and relate that with the computed Shapley value. For
example, in the last row of~\cref{tab:runex}, the prediction is 0, and
the \emph{irrelevant} feature 4 has a \emph{positive} Shapley value.
As a result, one might be tempted to believe that the irrelevant
feature 4 would contribute to \emph{changing} the value of the
prediction. This is of course incorrect, since an irrelevant feature
does not occur in \emph{any} CXp's (besides not occurring in any
AXp's) and so it is never necessary to changing the prediction.
The key point here is that irrelevant features are \emph{never}
necessary, neither to keep nor to change the prediction.

\section{Refuting Shapley Values for Explainability}
\label{sec:negres}

The purpose of this section is to prove that for arbitrary large
numbers of variables, there exist boolean functions and instances for
which the Shapley values exhibit the issues reported in recent
work~\cite{hms-corr23}, and detailed in~\cref{sec:cmp}.
(Instead of detailed proofs, this section describes the key ideas of
each proof. The detailed proofs are included in~\cref{sec:proofs}.)

Throughout this section, let $m$ be the number of variables of
the boolean functions we start from, and let $n$ denote the number of
variables of the functions we will be constructing. In this case, we
set $\fml{F}=\{1,\ldots,n\}$.
Furthermore, for the sake of simplicity, 
we opt to introduce the new features as the last features (e.g., feature $n$).
This choice does not affect the proof's argument in any way.
\begin{restatable}{proposition}{PropIRR}%
  \label{prop:irr}%
  For any $n\ge3$, there exist boolean functions defined on $n$
  variables, and at least one instance, which exhibit an
  issue~\cref{en:i1}, i.e.\ there exists an irrelevant feature
  $i\in\fml{F}$, such that $\sv(i)\not=0$.
\end{restatable}
\begin{proof}[Proof idea]
  The proof proposes to construct boolean functions, with an arbitrary
  number of variables (no smaller than 3), and the picking of an
  instance, such that a specific feature is irrelevant for the
  prediction, but its Shapley value is non-zero. To illustrate the
  construction, the example function from~\cref{ex:k1} is used (see
  also~\cref{ex:k1ex}).

  The construction works as follows. We pick two non-constant functions
  $\kappa_1(x_1,\ldots,x_m)$ and $\kappa_2(x_1,\ldots,x_m)$, defined
  on $m$ features, and such that:
  i) $\kappa_1\entails\kappa_2$ (which signifies that
  $\forall(\mbf{x}\in\mbb{F}).\kappa_1(\mbf{x})\limply\kappa_2(\mbf{x})$),
  and ii) $\kappa_1\neq\kappa_2$.
  Observe that $\kappa_1$ can be any boolean function defined on $m$
  variables, as long as $\kappa_2$ can also be defined.
  We then construct a new function by adding a new feature $n=m+1$, as
  follows:
  \[
  \kappa(x_1,\ldots,x_{m},x_{n})=\left\{
  \begin{array}{lcl}
    \kappa_1(x_1,\ldots,x_{m}) & \quad & \tn{if } x_n=0\\[3pt]
    \kappa_2(x_1,\ldots,x_{m}) & \quad & \tn{if } x_n=1\\
  \end{array}
  \right.
  \]
  %
  For the resulting function $\kappa$, we pick an instance
  $(\mbf{v},0)$ such that: i) $v_n=1$ and 
  ii) $\kappa_1(\mbf{v}_{1..m}) = \kappa_2(\mbf{v}_{1..m}) = 0$.
  The proof hinges on the fact that
  feature $n$ is irrelevant, but $\sv(n)\not=0$.
  
  For the function~\cref{ex:k1}, we set
  $\kappa_1(x_1,x_2)={x_1}\land{x_2}$ and
  $\kappa_1(x_1,x_2)={x_1}$.
  Thus, as shown in~\cref{ex:k1ex},
  $\kappa_{I1}(x_1,x_2,x_3)=(x_1\land{x_2}\land\neg{x}_3)\lor(x_1\land{x_3})$,
  which represents the function $\kappa(x_1,x_2,x_3)$.
  It is also clear that $\kappa_1\entails\kappa_2$.
  Moreover, and as~\cref{ex:k1ex} and~\cref{tab:runex} show, it is the
  case that feature 3 is irrelevant and $\sv(3)\not=0$. \qedhere
  %
\end{proof}
\begin{restatable}{proposition}{PropREL}%
  \label{prop:rel}%
  For any odd $n\ge3$, there exist boolean functions defined on $n$
  variables, and at least one instance, which exhibits an~\cref{en:i3}
  issue, i.e.~ for which there exists a relevant feature
  $i\in\fml{F}$, such that $\sv(i)=0$.
\end{restatable}
\begin{proof}[Proof idea]
  The proof proposes to construct boolean functions, with an arbitrary
  number of variables (no smaller than 3), and the picking of an
  instance, such that a specific feature is relevant for the
  prediction, but its Shapley value is zero. To illustrate the
  construction, the example function from~\cref{ex:k3} is used (see
  also~\cref{ex:k3ex}).

  The construction works as follows. We pick two non-constant functions
  $\kappa_1(x_1,\ldots,x_m)$ and $\kappa_2(x_{m+1},\ldots,x_{2m})$,
  each defined on $m$ features, where $\kappa_2$ corresponds to
  $\kappa_1$, but with a change of variables.
  Observe that $\kappa_1$ can be any boolean function.
  We then construct a new function, defined in terms of $\kappa_1$ and
  $\kappa_2$, by adding a new feature $n=2m+1$, as follows:
  \[
  \kappa(x_1,\ldots,x_{m},x_{m+1},\ldots,x_{2m},x_{n})=\left\{
  \begin{array}{lcl}
    \kappa_1(x_1,\ldots,x_{m}) & \quad & \tn{if } x_n=0\\[3pt]
    \kappa_2(x_{m+1},\ldots,x_{2m}) & \quad & \tn{if } x_n=1\\
  \end{array}
  \right.
  \]
  %
  For the resulting function $\kappa$, we pick an instance
  $(\mbf{v},1)$ such that: i) $v_n=1$, ii) $v_i = v_{m+i}$ for any $1 \le i \le m$,
  and iii) $\kappa_1(\mbf{v}_{1..m}) = \kappa_2(\mbf{v}_{m+1..2m}) = 1$.
  The proof hinges on the fact that feature $n$ is relevant, but $\sv(n)=0$.

  For the function~\cref{ex:k3}, we set
  $\kappa_1(x_1)={x_1}$ and
  $\kappa_1(x_2)={x_2}$. Thus, as shown in~\cref{ex:k3ex},
  $\kappa_{I3}(x_1,x_2,x_3)=(x_1\land\neg{x}_3)\lor(\neg{x}_2\land{x_3})$,
  which represents the function $\kappa(x_1,x_2,x_3)$.
  Moreover, and as~\cref{ex:k3ex} and~\cref{tab:runex} show, it is the
  case that feature 3 is relevant and $\sv(3)=0$.\qedhere
  %
\end{proof}
\begin{restatable}{proposition}{PropDisorder}%
  \label{prop:disorder}%
  For any even $n\ge4$, there exist boolean functions defined on $n$
  variables, and at least one instance, for which there exists an
  irrelevant feature $i_1\in\fml{F}$, such that $\sv(i_1)\neq0$, and a
  relevant feature $i_2\in\fml{F}\setminus\{i_1\}$, such that
  $\sv(i_2)=0$.
\end{restatable}
\begin{proof}[Proof idea]
  The proof proposes to construct boolean functions, with an arbitrary
  number of variables (no smaller than 4), and the picking of an
  instance, such that two specific features are such that one is
  relevant but has a Shapley value of 0, and the other one is
  irrelevant but has a non-zero Shapley values. To illustrate the
  construction, the example function from~\cref{ex:k4} is used (see
  also~\cref{ex:k4ex}).

  The construction works as follows. We pick two non-constant functions
  $\kappa_1(x_1,\ldots,x_m)$ and $\kappa_2(x_{m+1},\ldots,x_{2m})$,
  each defined on $m$ features, where $\kappa_2$ corresponds to
  $\kappa_1$, but with a change of variables.
  Also, observe that $\kappa_1$ can be any boolean function.
  We then construct a new function, defined in terms of $\kappa_1$ and
  $\kappa_2$, by adding two new features. We let the new features be
  $n-1$ and $n$, and so $n=2m+2$. The function is organized as
  follows:
  \[
  \kappa(\mbf{x}_{1..m},\mbf{x}_{m+1..2m},x_{n-1},x_{n})=\left\{
  \begin{array}{lcl}
    \kappa_1(\mbf{x}_{1..m})\land\kappa_2(\mbf{x}_{m+1..2m})
    & \quad & \tn{if } x_{n-1}=0\\[3pt]
    \kappa_1(\mbf{x}_{1..m}) & \quad & \tn{if } x_{n-1}=1\land{x_n}=0\\[3pt]
    \kappa_2(\mbf{x}_{m+1..2m}) & \quad & \tn{if } x_{n-1}=1\land{x_n}=1\\
  \end{array}
  \right.
  \]
  %
  For this function, we pick an instance
  $(\mbf{v},0)$ such that: i) $v_{n-1}=v_n=1$, ii) $v_i = v_{m+i}$ for any $1 \le i \le m$,
  and iii) $\kappa_1(\mbf{v}_{1..m}) = \kappa_2(\mbf{v}_{m+1..2m}) = 0$.
  The proof hinges on the fact that
  feature $n-1$ is irrelevant, feature $n$ is relevant, and
  $\sv(n-1)\not=0$ and $\sv(n)=0$. 

  For the function~\cref{ex:k4}, we set
  $\kappa_1(x_1)={x_1}$ and
  $\kappa_1(x_2)={x_2}$, Thus, as shown in~\cref{ex:k4ex},
  $\kappa_{I4}(x_1,x_2,x_3,x_4)=(x_1 \land x_2 \land \neg{x}_3) \lor
  (x_1 \land x_3 \land \neg{x}_4) \lor (x_2 \land x_3 \land x_4)$,
  which represents the function $\kappa(x_1,x_2,x_3,x_4)$.
  Moreover, and as~\cref{ex:k4ex} and~\cref{tab:runex} show, it is the
  case that feature 3 is irrelevant, feature 4 is relevant, and also
  $\sv(3)\not=0$ and $\sv(4)=0$.\qedhere
  %
\end{proof}
\begin{restatable}{proposition}{PropHighest}%
  \label{prop:highest}%
  For any $n\ge4$, there exists boolean functions defined on $n$
  variables, and at least one instance, for which there exists an
  irrelevant feature $i\in\fml{F}=\{1,\ldots,n\}$, such that
  $|\sv(i)|=\max\{|\sv(j)|\:|\,j\in\fml{F}\}$.
\end{restatable}
\begin{proof}[Proof idea]
  The proof proposes to construct boolean functions, with an arbitrary
  number of variables (no smaller than 4), and the picking of an
  instance, such that one specific feature is irrelevant but it has
  the Shapley value with the largest absolute values. To illustrate the
  construction, the example function from~\cref{ex:k5} is used (see
  also~\cref{ex:k5ex}).

  The construction works as follows. We pick one non-constant function
  $\kappa_1(x_1,\ldots,x_m)$, defined on $m$ features, such that:
  i) $\kappa_1$ predicts a specific point $\mbf{v}_{1..m}$ as 0, moreover,
  for any point $\mbf{x}_{1..m}$ such that $d_{H}(\mbf{x}_{1..m}, \mbf{v}_{1..m}) = 1$, $\kappa_1(\mbf{x}_{1..m}) = 1$,
  where $d_{H}(\cdot)$ denotes the Hamming distance.
  ii) and $\kappa_1$ predicts all the other points as 0.
  For example, let $\kappa_1(x_1,\ldots,x_m)=1$ iff $\sum_{i=1}^{m}\neg{x_1}=1$.
  We then construct a new function, defined in terms of $\kappa_1$, by
  adding one new feature. We let the new feature be $n$, and so
  $n=m+1$. The new function is organized as follows:
  \[
  \kappa(x_1,\ldots,x_{m},x_{n})=\left\{
  \begin{array}{lcl}
    0 & \quad & \tn{if } x_{n}=0\\[3pt]
    \kappa_1(x_1,\ldots,x_m) & \quad & \tn{if } x_{n}=1\\
  \end{array}
  \right.
  \]
  %
  For this function, we pick the instance
  $(\mbf{v},0)$ such that: i) $v_n=1$, 
  ii) $\mbf{v}_{1..m}$ is the only point within the Hamming ball
  and iii) $\kappa_1(\mbf{v}_{1..m}) = 0$.
  The proof hinges on the fact that feature $n$ is irrelevant, but
  $\forall(1\le{j}\le{m}).|\sv(j)|<|\sv(n)|$.

  For the function~\cref{ex:k5}, we set
  $\kappa_1(x_1,x_2,x_3)= (x_1 \land x_2 \land \neg{x}_3) \lor (x_1 \land x_3 \land \neg{x}_2) \lor (x_2 \land x_3 \land \neg{x}_1)$
  (i.e.\ the function takes value 1 when exactly one feature is 0).
  Thus, as shown in~\cref{ex:k4ex},
  $\kappa_{I5}(x_1,x_2,x_3,x_4)=((x_1 \land x_2 \land \neg{x}_3) \lor (x_1 \land x_3 \land \neg{x}_2) \lor (x_2 \land x_3 \land \neg{x}_1))\land{x_4}$,
  which represents the function $\kappa(x_1,x_2,x_3,x_4)$.
  Moreover, and as~\cref{ex:k5ex} and~\cref{tab:runex} show, it is the
  case that feature 4 is irrelevant and
  $\forall(1\le{j}\le{3}).|\sv(j)|<|\sv(4)|$.\qedhere
  %
\end{proof}

For~\cref{en:i2}, we can restate the previous result, but such
the functions constructed in the proof capture a more general family
of functions.

\begin{restatable}{proposition}{PropDisorderII}%
    \label{prop:disorder2}
    For any $n\ge4$, there exist boolean functions defined on $n$
    variables, and at least one instance, for which there exists an
    irrelevant feature $i_1\in\fml{F}$, and a relevant feature $i_2\in\fml{F}\setminus\{i_1\}$,
    such that $|\sv(i_1)| > |\sv(i_2)|$.
\end{restatable}



As noted above, for~\cref{prop:irr,prop:rel,prop:disorder}, the choice
of the starting function is fairly flexible. In contrast,
for~\cref{prop:highest}, we pick \emph{one} concrete function, which
represents a trivial lower bound. As a result, and with the exception
of \cref{en:i5}, we can prove the following (fairly loose) lower
bounds on the number of functions exhibiting the different issues.
\begin{restatable}{proposition}{PropLBs}%
  \label{prop:lbs}%
  For~\cref{prop:irr,prop:rel,prop:disorder},and \cref{prop:disorder2} the
  following are lower bounds on the numbers issues exhibiting the
  respective issues:
  \begin{enumerate}[nosep]
  \item For~\cref{prop:irr}, a lower bound on the number of functions
    exhibiting \cref{en:i1} is $2^{2^{(n-1)}}-n-3$.
  \item For~\cref{prop:rel}, a lower bound on the number of functions
    exhibiting \cref{en:i3} is $2^{2^{\sfrac{(n-1)}{2}}}-2$.
  \item For~\cref{prop:disorder}, a lower bound on the number of
    functions exhibiting \cref{en:i4} is $2^{2^{\sfrac{(n-2)}{2}}}-2$.
  \item For~\cref{prop:disorder2}, a lower bound on the number of
    functions exhibiting \cref{en:i2} is $2^{2^{n-2}-(n-2)-1}-1$.
  \end{enumerate}
\end{restatable}


%
%

\section{Conclusions} \label{sec:conc}

This paper gives theoretical arguments to the fact that Shapley values
for explainability can produce misleading information about the
relative importance of features. The paper distinguishes between the
features that occur in one or more of the irreducible rule-based
explanations, i.e.\ the \emph{relevant} features, from those that do
not occur in any irreducible rule-based explanation, i.e.\ the
\emph{irrelevant} features.
The paper proves that, for boolean functions with arbitrary number of
variables, irrelevant features can be deemed more important, given
their Shapley value, than relevant features.
Our results are also significant in practical deployment of
explainability solutions.
Indeed, misleading information about relative feature importance can 
induce human decision makers in error, by persuading them to look at
the wrong causes of predictions. 

One direction of research is to develop a better understanding of the
distributions of functions exhibiting one or more of the issues of
Shapley values.



\section*{Acknowledgments}
%
This work was supported by the AI Interdisciplinary Institute ANITI,
funded by the French program ``Investing for the Future -- PIA3''
under Grant agreement no.\ ANR-19-PI3A-0004,
and by the H2020-ICT38 project COALA ``Cognitive Assisted agile 
manufacturing for a Labor force supported by trustworthy Artificial
intelligence''.
This work was motivated in part by discussions with several colleagues
including L.~Bertossi, A.~Ignatiev, N.~Narodytska, M.~Cooper, Y.~Izza,
R.\ Passos, J.\ Planes and N.~Asher. 
%
%

\section*{Disclosure of Past Submission}
This work was submitted to the 2023 NeurIPS conference, but it was
not accepted for publication.
Incidentally, the best-known work on using Shapley values for
explainability, which this paper categorically refutes, was published
earlier at the NeurIPS conference~\cite{lundberg-nips17}.
The reviews and the public discussion of the 2023 NeurIPS submission
are visible on
\href{https://openreview.net/forum?id=wjqT8OBm0y}{openreview.net}.
%
The reading of the NeurIPS 2023 reviews and the public discussion is
recommended.
Despite this paper not having been accepted for publication at the
2023 NeurIPS conference, we have continued to gather evidence on the
serious shortcomings of existing definitions of SHAP scores; this is
illustrated in recent
publications~\cite{msh-cacm24,hms-ijar24,hms-corr23c}.

%

\newtoggle{mkbbl}

\settoggle{mkbbl}{false}

\iftoggle{mkbbl}{
  \bibliographystyle{abbrv}
  \bibliography{refs,xtra}
}{

}

%
%
\clearpage
\appendix
\section{Detailed Proofs} \label{sec:proofs}

\PropDualTwo*

\begin{proof}
  This result is a consequence of minimal-hitting set duality between
  AXp's and CXp's, proved elsewhere~\cite{inams-aiia20}.
\end{proof}

\PropIfiveToItwo*

\begin{proof}
  Given a classifier with classification function $\kappa$ and an
  instance $(\mbf{v},c)$, it is plain that the set of features
  $\fml{F}$ represents a WAXp.
  Furthermore, since the classification function is assumed not to be
  constant, then there must exist some AXp that is not the empty set.
  Thus, such AXp contains at least one relevant feature, say
  $i_{\mathrm{rel}}\in\fml{F}$.
  Moreover, if \cref{en:i5} holds, then there exists an irrelevant
  $i_{\mathrm{irr}}\in\fml{F}\setminus\{i_{\mathrm{rel}}\}$ with the
  largest absolute Shapley value.
  Therefore, it is the case that for feature $i_{\mathrm{rel}}$, its
  absolute Shapley value is smaller than that of irrelevant feature
  $i_{\mathrm{irr}}$. As a result, the function also exhibits issue
  \cref{en:i2}.\qedhere
\end{proof}

\PropIRR*

\begin{proof}
Consider two classifiers $\fml{M}_1$ and $\fml{M}_2$ implementing non-constant boolean functions $\kappa_1$ and $\kappa_2$, respectively.
These functions are defined on the set of features $\fml{F}' = \{1, \dots, m\}$,
and such that $\kappa_1 \models \kappa_2$ but $\kappa_1 \neq \kappa_2$.
Consider the set of features $\fml{F} = \fml{F}' \cup \{n\}$, we construct a new classifier $\fml{M}$ by combining $\fml{M}_1$ and $\fml{M}_2$.
The classifier $\fml{M}$ is characterized by the boolean function defined as follows:
\begin{equation}
\kappa(x_1, \ldots, x_m, x_n) :=
\left \{
\begin{array}{lcl}
\kappa_1(x_1, \ldots, x_m) & \quad & \tn{if~$x_n=0$}\\[3pt]
\kappa_2(x_1, \ldots, x_m) & \quad & \tn{if~$x_n=1$}\\
\end{array}
\right.
\end{equation}
Choose a $m$-dimensional point $\mbf{v}_{1..m}$ such that $\kappa_1(\mbf{v}_{1..m}) = \kappa_2(\mbf{v}_{1..m}) = 0$,
and extend $\mbf{v}_{1..m}$ with $v_n=1$.
Then for the $n$-dimensional point $\mbf{v}_{1..n} = (\mbf{v}_{1..m}, 1)$, we have $\kappa(\mbf{v}_{1..n}) = 0$.

To simplify the notation, we will use $\mbf{x}'$ to denote an
arbitrary $n$-dimensional point $\mbf{x}_{1..n}$.
Additionally, we will use $\mbf{y}$ to denote an arbitrary $m$-dimensional point $\mbf{x}_{1..m}$.
For any subset $\fml{S} \subseteq \fml{F}'$, we have:
\begin{align}
&\phi(\fml{S}\cup\{n\};\fml{M},\mbf{v}_{1..n}) - \phi(\fml{S};\fml{M},\mbf{v}_{1..n}) \\ \nonumber
&= \left( \frac{1}{2^{|(\fml{F}'\cup\{n\})\setminus(\fml{S}\cup\{n\})|}}\sum_{\mbf{x}'\in\Upsilon(\fml{S}\cup\{n\};\mbf{v}_{1..n})}\kappa(\mbf{x}') \right)
- \left( \frac{1}{2^{|(\fml{F}'\cup\{n\})\setminus\fml{S}|}}\sum_{\mbf{x}'\in\Upsilon(\fml{S};\mbf{v}_{1..n})}\kappa(\mbf{x}') \right) \\ \nonumber
&= \left( \frac{1}{2^{|\fml{F}'\setminus\fml{S}|}}\sum_{\mbf{y}\in\Upsilon(\fml{S};\mbf{v}_{1..m})}\kappa_2(\mbf{y}) \right)
- \frac{1}{2^{|\fml{F}'\setminus\fml{S}|+1}}
\left( \sum_{\mbf{x}'\in\Upsilon(\fml{S};(\mbf{v}_{1..m},1))}\kappa(\mbf{x}') + \sum_{\mbf{x}'\in\Upsilon(\fml{S};(\mbf{v}_{1..m},0))}\kappa(\mbf{x}') \right) \\ \nonumber
&= \frac{1}{2^{|\fml{F}'\setminus\fml{S}|}} \left( \sum_{\mbf{y}\in\Upsilon(\fml{S};\mbf{v}_{1..m})}\kappa_2(\mbf{y})
- \frac{1}{2} \times \sum_{\mbf{y}\in\Upsilon(\fml{S};\mbf{v}_{1..m})}\kappa_2(\mbf{y}) 
- \frac{1}{2} \times \sum_{\mbf{y}\in\Upsilon(\fml{S};\mbf{v}_{1..m})}\kappa_1(\mbf{y}) \right) \\ \nonumber
&= \frac{1}{2} \times \frac{1}{2^{|\fml{F}'\setminus\fml{S}|}} 
\left( \sum_{\mbf{y}\in\Upsilon(\fml{S};\mbf{v}_{1..m})}\kappa_2(\mbf{y}) - \sum_{\mbf{y}\in\Upsilon(\fml{S};\mbf{v}_{1..m})}\kappa_1(\mbf{y}) \right)
\end{align}
Given that $\kappa_1 \models \kappa_2$ but $\kappa_1 \neq \kappa_2$,
it follows that for any points $\mbf{y}\in\Upsilon(\fml{S};\mbf{v}_{1..m})$, if $\kappa_1(\mbf{y})=1$ then $\kappa_2(\mbf{y})=1$.
In other words, if $\kappa_2(\mbf{y})=0$ then $\kappa_1(\mbf{y})=0$.
Moreover, there are cases where the following inequality holds:
$\sum_{\mbf{y}\in\Upsilon(\fml{S};\mbf{v}_{1..m})}\kappa_2(\mbf{y}) - \sum_{\mbf{y}\in\Upsilon(\fml{S};\mbf{v}_{1..m})}\kappa_1(\mbf{y}) > 0$.
Hence, $\sv(n) \neq 0$.

To prove that the feature $n$ is irrelevant, 
we assume the contrary, i.e., that $n$ is relevant, and $\fml{X}$ is an AXp of $\fml{M}$ for the point $\mbf{v}_{1..n}$ such that $n \in \fml{X}$. 
This means we fix the variable $x_n$ to the value $v_n$, and, based on the definition of AXp, we only select the points that $\fml{M}_2$ predicts as 0.
Since $\kappa_2(\mbf{y})=0$ implies that $\kappa_1(\mbf{y})=0$, 
removing feature $n$ from $\fml{X}$ means that $\fml{X} \setminus {n}$ will not include any points predicted as 1 by either $\fml{M}_1$ or $\fml{M}_2$. 
Thus, $\fml{X} \setminus {n}$ remains an AXp of $\fml{M}$ for the point $\mbf{v}_{1..n}$, leading to a contradiction.
Thus, feature $n$ is irrelevant.
\qedhere
\end{proof}

\PropREL*

\begin{proof}
Given a classifier $\fml{M}_1$ implementing a non-constant boolean function $\kappa_1$ defined on the set of features
$\fml{F}_1 = \{1, \dots, m\}$. 
We can replace each $x_i$ of $\kappa_1$ with a new variable $x_{m+i}$ to obtain a new function $\kappa_2$,
defined on a new set of features $\fml{F}_2 = \{m+1, \dots, 2m\}$.
Importantly, $\kappa_2$ is independent of $\kappa_1$ as $\fml{F}_1$ and $\fml{F}_2$ are disjoint.
Let $\fml{F} = \fml{F}_1 \cup \fml{F}_2 \cup \{n\}$,
we build a new classifier $\fml{M}$ characterized by the boolean function defined as follows:
\begin{equation}
\kappa(x_1, \dots, x_m, x_{m+1}, \dots, x_{2m}, x_n) :=
\left \{
\begin{array}{lcl}
\kappa_1(x_1, \dots, x_m) & \quad & \tn{if~$x_n=0$}\\[3pt]
\kappa_2(x_{m+1}, \dots, x_{2m}) & \quad & \tn{if~$x_n=1$}\\
\end{array}
\right.
\end{equation}
Choose $m$-dimensional points $\mbf{v}_{1..m}$ and $\mbf{v}_{m+1..2m}$ 
such that $v_{i} = v_{m+i}$ for any $1 \le i \le m$, and $\kappa_1(\mbf{v}_{1..m}) = \kappa_2(\mbf{v}_{m+1..2m}) = 1$.
Let $\mbf{v}_{1..n} = (\mbf{v}_{1..m}, \mbf{v}_{m+1..2m}, 1)$ be a $n$-dimensional point such that $\kappa(\mbf{v}_{1..n}) = 1$.
Moreover, let $\fml{F}' = \fml{F}_1 \cup \fml{F}_2$.

To simplify the notations, we will use $\mbf{u}$ to denote $\mbf{v}_{1..m}$ and $\mbf{w}$ to denote $\mbf{v}_{m+1..2m}$,
furthermore, we will use $\mbf{x}'$ to denote an arbitrary $n$-dimensional point $\mbf{x}_{1..n}$,
and $\mbf{y}$ to denote an arbitrary $m$-dimensional point $\mbf{x}_{1..m}$,
and $\mbf{z}$ to denote an arbitrary $m$-dimensional point $\mbf{x}_{m+1..2m}$.
For any subset $\fml{S} \subseteq \fml{F}'$, let $\{\fml{S}_{1}, \fml{S}_{2}\}$ be a partition of $\fml{S}$
such that $\fml{S}_{1} \subseteq \fml{F}_1 \land \fml{S}_{2} \subseteq \fml{F}_2$, then:
\begin{align}
&\phi(\fml{S}\cup\{n\};\fml{M},\mbf{v}_{1..n}) - \phi(\fml{S};\fml{M},\mbf{v}_{1..n}) \\ \nonumber
&= \left( \frac{1}{2^{|(\fml{F}'\cup\{n\})\setminus(\fml{S}\cup\{n\})|}}\sum_{\mbf{x}'\in\Upsilon(\fml{S}\cup\{n\};\mbf{v}_{1..n})}\kappa(\mbf{x}') \right)
- \left( \frac{1}{2^{|(\fml{F}'\cup\{n\})\setminus\fml{S}|}}\sum_{\mbf{x}'\in\Upsilon(\fml{S};\mbf{v}_{1..n})}\kappa(\mbf{x}') \right) \\ \nonumber
&= \frac{1}{2^{|\fml{F}'\setminus\fml{S}|}} \left( \sum_{\mbf{x}'\in\Upsilon(\fml{S};(\mbf{u},\mbf{w},1))}\kappa(\mbf{x}') - \frac{1}{2} \times 
\sum_{\mbf{x}'\in\Upsilon(\fml{S};(\mbf{u},\mbf{w},1))}\kappa(\mbf{x}') - \frac{1}{2} \times \sum_{\mbf{x}'\in\Upsilon(\fml{S};(\mbf{u},\mbf{w},0))}\kappa(\mbf{x}') \right)
\\ \nonumber
&= \frac{1}{2} \times \frac{1}{2^{|\fml{F}'\setminus\fml{S}|}} \left( \sum_{\mbf{x}'\in\Upsilon(\fml{S};(\mbf{u},\mbf{w},1))}\kappa(\mbf{x}')
- \sum_{\mbf{x}'\in\Upsilon(\fml{S};(\mbf{u},\mbf{w},0))}\kappa(\mbf{x}') \right) \\ \nonumber
&= \frac{1}{2} \times \frac{1}{2^{|\fml{F}'\setminus\fml{S}|}}
\left( 2^{|\fml{F}_1\setminus\fml{S}_{1}|} \times \sum_{\mbf{z}\in\Upsilon(\fml{S}_2;\mbf{w})}\kappa_2(\mbf{z})
- 2^{|\fml{F}_2\setminus\fml{S}_2|} \times \sum_{\mbf{y}\in\Upsilon(\fml{S}_1;\mbf{u})}\kappa_1(\mbf{y}) \right)
\end{align}
For any $\{\fml{S}_1, \fml{S}_2\}$, we can construct a unique new partition $\{\fml{S}'_1, \fml{S}'_2\}$
by replacing any $i \in \fml{S}_1$ with $m+i$ and any $m+i \in \fml{S}_2$ with $i$.
Let $\fml{S}' = \fml{S}'_{1} \cup \fml{S}'_{2}$, then we have:
\begin{align}
&\phi(\fml{S}'\cup\{n\};\fml{M},\mbf{v}_{1..n}) - \phi(\fml{S}';\fml{M},\mbf{v}_{1..n}) \\ \nonumber
&= \frac{1}{2} \times \frac{1}{2^{|\fml{F}'\setminus\fml{S}'|}}
\left( 2^{|\fml{F}_1\setminus\fml{S}'_2|} \times \sum_{\mbf{z}\in\Upsilon(\fml{S}'_1;\mbf{w})}\kappa_2(\mbf{z})
- 2^{|\fml{F}_2\setminus\fml{S}'_1|} \sum_{\mbf{y}\in\Upsilon(\fml{S}'_2;\mbf{u})}\kappa_1(\mbf{y}) \right)
\end{align}
Besides, we have:
\[
2^{|\fml{F}_1\setminus\fml{S}_1|} \times \sum_{\mbf{z}\in\Upsilon(\fml{S}_2;\mbf{z})}\kappa_2(\mbf{z}) = 
2^{|\fml{F}_2\setminus\fml{S}'_1|} \sum_{\mbf{y}\in\Upsilon(\fml{S}'_2;\mbf{u})}\kappa_1(\mbf{y})
\]
and 
\[
2^{|\fml{F}_2\setminus\fml{S}_2|} \times \sum_{\mbf{y}\in\Upsilon(\fml{S}_1;\mbf{u})}\kappa_1(\mbf{y}) = 
2^{|\fml{F}_1\setminus\fml{S}'_2|} \times \sum_{\mbf{z}\in\Upsilon(\fml{S}'_1;\mbf{w})}\kappa_2(\mbf{z})
\]
which means: 
\[
\phi(\fml{S}\cup\{n\};\fml{M},\mbf{v}_{1..n}) - \phi(\fml{S};\fml{M},\mbf{v}_{1..n}) 
= - (\phi(\fml{S}'\cup\{n\};\fml{M},\mbf{v}_{1..n}) - \phi(\fml{S}';\fml{M},\mbf{v}_{1..n}))
\]
note that 
$\frac{|\fml{S}|!(|\fml{F}|-|\fml{S}|-1)!}{|\fml{F}|!} = \frac{|\fml{S}'|!(|\fml{F}|-|\fml{S}'|-1)!}{|\fml{F}|!}$.
Hence, for any subset $\fml{S}$, there is a unique subset $\fml{S}'$ that can cancel its effect, from which we can derive that
$\sv(n) = 0$.
However, $n$ is a relevant feature.
To find an AXp containing $n$, we remove all features in $\fml{F}_1$, and keep only feature $n$ along with all features in $\fml{F}_2$.
This makes feature $n$ critical to the change in the prediction of $\fml{M}$.
Next, we compute an AXp $\fml{X}$ of $\fml{M}_2$ under the point $\mbf{v}_{m+1..2m}$.
Finally, $\fml{X} \cup \{n\}$ is an AXp of the classifier $\fml{M}$ for the point $\mbf{v}_{1..n}$.
\qedhere
\end{proof}

\PropDisorder*

\begin{proof}
Given a classifier $\fml{M}_1$ implementing a non-constant boolean function $\kappa_1$ defined on
the set of features $\fml{F}_1 = \{1, \dots, m\}$,
we can construct a new classifier $\fml{M}$ characterized by the boolean function defined as follows:
\begin{equation}
\kappa(\mbf{x}_{1..m},\mbf{x}_{m+1..2m},x_{n-1},x_{n}) :=
\left \{
\begin{array}{lcl}
\kappa_1(\mbf{x}_{1..m})\land\kappa_2(\mbf{x}_{m+1..2m}) & \quad & \tn{if~$x_{n-1}=0$}\\[3pt]
\kappa_1(\mbf{x}_{1..m}) & \quad & \tn{if~$x_{n-1}=1 \land x_n = 0$}\\[3pt]
\kappa_2(\mbf{x}_{m+1..2m}) & \quad & \tn{if~$x_{n-1} = 1 \land x_n = 1$}\\
\end{array}
\right.
\end{equation}
where function $\kappa_2$ is obtained by replacing every $x_i$ of $\kappa_1$ with a new variable $x_{m+i}$.
$\kappa_2$ is defined on a new set of features $\fml{F}_2 = \{m+1, \dots, 2m\}$ and is independent of $\kappa_1$.
Moreover, $\fml{M}$ is defined on the feature set $\fml{F} = \fml{F}_1 \cup \fml{F}_2 \cup \{n-1, n\}$.
Note that $\kappa_1 \land \kappa_2 \models (\neg x_n \land \kappa_1) \lor (x_n \land \kappa_2)$,
this can be proved using the consensus theorem~
\footnote{The consensus theorem is the identity $(x \land y) \lor (\neg x \land z) = (x \land y) \lor (\neg x \land z) \lor (y \land z)$, 
see \cite{crama2011boolean} Chapter 3}.

Choose $m$-dimensional points $\mbf{v}_{1..m}$ and $\mbf{v}_{m+1..2m}$ 
such that $v_{i} = v_{m+i}$ for any $1 \le i \le m$, and $\kappa_1(\mbf{v}_{1..m}) = \kappa_2(\mbf{v}_{m+1..2m}) = 0$.
Let $\mbf{v}_{1..n} = (\mbf{v}_{1..m}, \mbf{v}_{m+1..2m}, 1, 1)$ be a $n$-dimensional point such that $\kappa(\mbf{v}_{1..n}) = 0$.
Moreover, let $\fml{F}' = \fml{F}_1 \cup \fml{F}_2$.

To simplify the notations, we will use $\mbf{u}$ to denote $\mbf{v}_{1..m}$ and $\mbf{w}$ to denote $\mbf{v}_{m+1..2m}$,
furthermore, we will use $\mbf{x}'$ to denote an arbitrary $n$-dimensional point $\mbf{x}_{1..n}$,
and $\mbf{y}$ to denote an arbitrary $m$-dimensional point $\mbf{x}_{1..m}$,
and $\mbf{z}$ to denote an arbitrary $m$-dimensional point $\mbf{x}_{m+1..2m}$.

According to the proof of \cref{prop:irr}, $\sv(n-1) \neq 0$ but feature $n-1$ is irrelevant.
Next, we show that $\sv(n) = 0$ but the feature $n$ is relevant.
For any subset $\fml{S} \subseteq \fml{F}'$, let $\{\fml{S}_{1}, \fml{S}_{2}\}$ be a partition of $\fml{S}$
such that $\fml{S}_{1} \subseteq \fml{F}_1 \land \fml{S}_{2} \subseteq \fml{F}_2$.
\begin{enumerate}
\item Consider any subset $\fml{S} \cup \{n-1\}$, then:
\begin{align}
&\phi(\fml{S}\cup\{n-1,n\};\fml{M},\mbf{v}_{1..n}) - \phi(\fml{S}\cup\{n-1\};\fml{M},\mbf{v}_{1..n}) \\ \nonumber
&= \left( \frac{1}{2^{|\fml{F}'\setminus\fml{S}|}}\sum_{\mbf{x}'\in\Upsilon(\fml{S}\cup\{n-1,n\};\mbf{v}_{1..n})}\kappa(\mbf{x}') \right)
- \left( \frac{1}{2^{|\fml{F}'\setminus\fml{S}|+1}}\sum_{\mbf{x}'\in\Upsilon(\fml{S}\cup\{n-1\};\mbf{v}_{1..n})}\kappa(\mbf{x}') \right) \\ \nonumber
&= \frac{1}{2} \times \frac{1}{2^{|\fml{F}'\setminus\fml{S}|}}
\left( \sum_{\mbf{x}'\in\Upsilon(\fml{S}\cup\{n-1,n\};(\mbf{u},\mbf{w}, 1, 1))}\kappa(\mbf{x}')
- \sum_{\mbf{x}'\in\Upsilon(\fml{S}\cup\{n-1\};(\mbf{u},\mbf{w}, 1, 0))}\kappa(\mbf{x}') \right) \\ \nonumber
&= \frac{1}{2} \times \frac{1}{2^{|\fml{F}'\setminus\fml{S}|}} 
\left( 2^{|\fml{F}_1\setminus\fml{S}_{1}|} \times \sum_{\mbf{z}\in\Upsilon(\fml{S}_2;\mbf{w})}\kappa_2(\mbf{z})
- 2^{|\fml{F}_2\setminus\fml{S}_{2}|} \times \sum_{\mbf{y}\in\Upsilon(\fml{S}_1;\mbf{u})}\kappa_1(\mbf{y}) \right)
\end{align}
According to the proof of \cref{prop:rel},
there is a unique subset $\fml{S}'$ such that $|\fml{S}| = |\fml{S}'|$ and
$\phi(\fml{S}\cup\{n-1,n\};\fml{M},\mbf{v}_{1..n}) - \phi(\fml{S}\cup\{n-1\};\fml{M},\mbf{v}_{1..n}) 
= - (\phi(\fml{S}'\cup\{n-1,n\};\fml{M},\mbf{v}_{1..n}) - \phi(\fml{S}'\cup\{n-1\};\fml{M},\mbf{v}_{1..n}))$.
\item Consider any subset $\fml{S}\subseteq\fml{F}'$, then:
\begin{align}
&\phi(\fml{S}\cup\{n\};\fml{M},\mbf{v}_{1..n}) - \phi(\fml{S};\fml{M},\mbf{v}_{1..n}) \\ \nonumber
&= \left( \frac{1}{2^{|\fml{F}'\setminus\fml{S}|+1}}\sum_{\mbf{x}'\in\Upsilon(\fml{S}\cup\{n\};\mbf{v}_{1..n})}\kappa(\mbf{x}') \right)
- \left( \frac{1}{2^{|\fml{F}'\setminus\fml{S}|+2}}\sum_{\mbf{x}'\in\Upsilon(\fml{S};\mbf{v}_{1..n})}\kappa(\mbf{x}') \right) \\ \nonumber
&= \frac{1}{2^{|\fml{F}'\setminus\fml{S}|+1}} \left( \sum_{\mbf{x}'\in\Upsilon(\fml{S}\cup\{n\};(\mbf{u},\mbf{w},1,1))}\kappa(\mbf{x}')
+ \sum_{\mbf{x}'\in\Upsilon(\fml{S}\cup\{n\};(\mbf{u},\mbf{w},0,1))}\kappa(\mbf{x}') \right) \\ \nonumber
&- \frac{1}{2^{|\fml{F}'\setminus\fml{S}|+2}} \left( \sum_{\mbf{x}'\in\Upsilon(\fml{S};(\mbf{u},\mbf{w},1,1))}\kappa(\mbf{x}')
+ \sum_{\mbf{x}'\in\Upsilon(\fml{S};(\mbf{u},\mbf{w},0,1))}\kappa(\mbf{x}') \right) \\ \nonumber
&- \frac{1}{2^{|\fml{F}'\setminus\fml{S}|+2}} \left( \sum_{\mbf{x}'\in\Upsilon(\fml{S};(\mbf{u},\mbf{w},1,0))}\kappa(\mbf{x}') + 
\sum_{\mbf{x}'\in\Upsilon(\fml{S};(\mbf{u},\mbf{w},0,0))}\kappa(\mbf{x}') \right) \\ \nonumber
&= \frac{1}{4} \times \frac{1}{2^{|\fml{F}'\setminus\fml{S}|}} 
\left( 2^{|\fml{F}_1\setminus\fml{S}_{1}|} \times \sum_{\mbf{z}\in\Upsilon(\fml{S}_2;\mbf{w})}\kappa_2(\mbf{z})
- 2^{|\fml{F}_2\setminus\fml{S}_{2}|} \times \sum_{\mbf{y}\in\Upsilon(\fml{S}_1;\mbf{u})}\kappa_1(\mbf{y})\right)
\end{align}
Likewise,
we can find a unique subset $\fml{S}'$ to cancel the effect of $\phi(\fml{S}\cup\{n\};\fml{M},\mbf{v}_{1..n}) - \phi(\fml{S};\fml{M},\mbf{v}{1..n})$.
\end{enumerate}
Therefore, $\sv(n) = 0$.
To prove that the feature $n$ is relevant, we compute an AXp containing the feature $n$.
First, we free all features in $\fml{F}_1$ and the feature $n-1$, while keeping all features in $\fml{F}_2$ and the feature $n$.
This makes feature $n$ critical to the change in the prediction of $\fml{M}$.
Next, we compute an AXp $\fml{X}$ of $\fml{M}_2$ under the point $\mbf{v}_{m+1..2m}$.
Finally, we can conclude that $\fml{X} \cup \{n\}$ is an AXp of $\fml{M}$ under the point $\mbf{v}_{1..n}$.
\qedhere
\end{proof}

\PropHighest*

\begin{proof}
Given a classifier $\fml{M}_1$ implementing a non-constant boolean function $\kappa_1$
defined on the set of variables $\fml{F}' = \{1, \dots, m\}$ where $m \ge 3$, and
satisfies the following conditions:
\begin{enumerate}
\item
$\kappa_1$ predicts a specific point $\mbf{v}_{1..m}$ as 0. Furthermore,
for any point $\mbf{x}_{1..m}$ such that $d_{H}(\mbf{x}_{1..m}, \mbf{v}_{1..m}) = 1$,
where $d_{H}(\cdot)$ denotes the Hamming distance, we have $\kappa_1(\mbf{x}_{1..m}) = 1$.
\item
$\kappa_1$ predicts all the other points as 0.
\end{enumerate}
For example, $\kappa_1$ can be the function $\sum_{i=1}^{m}\neg{x_1}=1$,
which predicts the point $\mbf{1}_{1..m}$ as 0 and all points around this point with a Hamming distance of 1 as 1.
Based on $\kappa_1$, we can build a new classifier $\fml{M}$ characterized by the boolean function defined as follows:
\begin{equation}
\kappa(x_1,\dots,x_m, x_n) :=
\left \{
\begin{array}{lcl}
0 & \quad & \tn{if~$x_n=0$}\\
\kappa_1(x_1,\dots,x_m) & \quad & \tn{if~$x_n=1$}\\[3pt]
\end{array}
\right.
\end{equation}
Select the $m$-dimensional point $\mbf{v}_{1..m}$ from our Hamming ball such that $\kappa_1(\mbf{v}_{1..m}) = 0$
(note that only one such point exists),
and extend $\mbf{v}_{1..m}$ with $v_n=1$.
Then for the $n$-dimensional point $\mbf{v}_{1..n} = (\mbf{v}_{1..m}, 1)$, we have $\kappa(\mbf{v}_{1..n}) = 0$.
Applying the same reasoning presented in the proof of \cref{prop:irr}, we can deduce that feature $n$ is irrelevant.

For simplicity, we will use $\mbf{x}'$ to denote an arbitrary $n$-dimensional point $\mbf{x}_{1..n}$,
and $\mbf{y}$ to denote an arbitrary $m$-dimensional point $\mbf{x}_{1..m}$.
More importantly, for $\kappa_1$ and any subset $\fml{S}\subseteq\fml{F}'$, we have:
\[
\sum_{\mbf{y}\in\Upsilon(\fml{S};\mbf{v}_{1..m})}\kappa_1(\mbf{y}) = m-|\fml{S}|
\]
\begin{enumerate}
\item
For the feature $n$ and an arbitrary subset $\fml{S} \subseteq \fml{F}'$, we have:
\begin{align}
&\phi(\fml{S}\cup\{n\};\fml{M},\mbf{v}_{1..n}) - \phi(\fml{S};\fml{M},\mbf{v}_{1..n}) \\ \nonumber
&= \frac{1}{2^{|(\fml{F}'\cup\{n\})\setminus(\fml{S}\cup\{n\})|}} \sum_{\mbf{x}'\in\Upsilon(\fml{S}\cup\{n\};\mbf{v}_{1..n})}\kappa(\mbf{x}')
- \frac{1}{2^{|(\fml{F}'\cup\{n\})\setminus\fml{S}|}} \sum_{\mbf{x}'\in\Upsilon(\fml{S};\mbf{v}_{1..n})}\kappa(\mbf{x}') \\ \nonumber
&= \frac{1}{2^{|\fml{F}'\setminus\fml{S}|}} \sum_{\mbf{x}'\in\Upsilon(\fml{S}\cup\{n\};\mbf{v}_{1..n})}\kappa(\mbf{x}')
- \frac{1}{2^{|\fml{F}'\setminus\fml{S}|+1}} \sum_{\mbf{x}'\in\Upsilon(\fml{S};\mbf{v}_{1..n})}\kappa(\mbf{x}') \\ \nonumber
&= \frac{1}{2} \times \frac{1}{2^{|\fml{F}'\setminus\fml{S}|}} \sum_{\mbf{y}\in\Upsilon(\fml{S};\mbf{v}_{1..m})}\kappa_1(\mbf{y}) \\ \nonumber
&= \frac{1}{2} \phi(\fml{S};\fml{M}_1,\mbf{v}_{1..m}) \\ \nonumber
&= \frac{1}{2} \times \frac{m-|\fml{S}|}{2^{m-|\fml{S}|}}
\end{align}
This means $\sv(n) > 0$.
Besides, the unique minimal value of $\phi(\fml{S}\cup\{n\};\fml{M},\mbf{v}_{1..n}) - \phi(\fml{S};\fml{M},\mbf{v}_{1..n})$ is 0
when $\fml{S} = \fml{F}'$.
\item
For a feature $j \neq n$, consider an arbitrary subset $\fml{S}\subseteq\fml{F}'\setminus\{j\}$ and the feature $n$, we have:
\begin{align}
&\phi(\fml{S}\cup\{j,n\};\fml{M},\mbf{v}_{1..n}) - \phi(\fml{S}\cup\{n\};\fml{M},\mbf{v}_{1..n}) \\ \nonumber
&= \frac{1}{2^{|(\fml{F}'\cup\{n\})\setminus(\fml{S}\cup\{j,n\})|}} \sum_{\mbf{x}'\in\Upsilon(\fml{S}\cup\{j,n\};\mbf{v}_{1..n})}\kappa(\mbf{x}')
- \frac{1}{2^{|(\fml{F}'\cup\{n\})\setminus(\fml{S}\cup\{n\})|}} \sum_{\mbf{x}'\in\Upsilon(\fml{S}\cup\{n\};\mbf{v}_{1..n})}\kappa(\mbf{x}') \\ \nonumber
&= \frac{1}{2^{|\fml{F}'\setminus(\fml{S}\cup\{j\})|}} \sum_{\mbf{y}\in\Upsilon(\fml{S}\cup\{j\};\mbf{v}_{1..m})}\kappa_1(\mbf{y})
- \frac{1}{2^{|\fml{F}'\setminus\fml{S}|}} \sum_{\mbf{y}\in\Upsilon(\fml{S};\mbf{v}_{1..m})}\kappa_1(\mbf{y}) \\ \nonumber
&= \phi(\fml{S}\cup\{j\};\fml{M}_1,\mbf{v}_{1..m}) - \phi(\fml{S};\fml{M}_1,\mbf{v}_{1..m}) \\ \nonumber
&= \frac{m-|\fml{S}|-1}{2^{m-|\fml{S}|-1}} - \frac{m-|\fml{S}|}{2^{m-|\fml{S}|}} \\ \nonumber
&= \frac{m-|\fml{S}|-2}{2^{m-|\fml{S}|}}
\end{align}
In this case, 
$\phi(\fml{S}\cup\{j,n\};\fml{M},\mbf{v}_{1..n}) - \phi(\fml{S}\cup\{n\};\fml{M},\mbf{v}_{1..n}) = -\frac{1}{2}$ if $|\fml{S}| = m-1$, which is its unique minimal value.
$\phi(\fml{S}\cup\{j,n\};\fml{M},\mbf{v}_{1..n}) - \phi(\fml{S}\cup\{n\};\fml{M},\mbf{v}_{1..n}) = 0$ if $|\fml{S}| = m-2$,
and $\phi(\fml{S}\cup\{j,n\};\fml{M},\mbf{v}_{1..n}) - \phi(\fml{S}\cup\{n\};\fml{M},\mbf{v}_{1..n}) > 0$ if $|\fml{S}| < m-2$.
\item
Moreover, for a feature $j \neq n$, consider an arbitrary subset $\fml{S}\subseteq\fml{F}'\setminus\{j\}$ and without the feature $n$, we have:
\begin{align}
&\phi(\fml{S}\cup\{j\};\fml{M},\mbf{v}_{1..n}) - \phi(\fml{S};\fml{M},\mbf{v}_{1..n}) \\ \nonumber
&= \frac{1}{2^{|(\fml{F}'\cup\{n\})\setminus(\fml{S}\cup\{j\})|}} \sum_{\mbf{x}'\in\Upsilon(\fml{S}\cup\{j\};\mbf{v}_{1..n})}\kappa(\mbf{x}')
- \frac{1}{2^{|(\fml{F}'\cup\{n\})\setminus\fml{S}|}} \sum_{\mbf{x}'\in\Upsilon(\fml{S};\mbf{v}_{1..n})}\kappa(\mbf{x}') \\ \nonumber
&= \frac{1}{2^{|\fml{F}'\setminus(\fml{S}\cup\{j\})|+1}} \sum_{\mbf{y}\in\Upsilon(\fml{S}\cup\{j\};\mbf{v}_{1..m})}\kappa_1(\mbf{y})
- \frac{1}{2^{|\fml{F}'\setminus\fml{S}|+1}} \sum_{\mbf{y}\in\Upsilon(\fml{S};\mbf{v}_{1..m})}\kappa_1(\mbf{y}) \\ \nonumber
&= \frac{1}{2} (\phi(\fml{S}\cup\{j\};\fml{M}_1,\mbf{v}_{1..m}) - \phi(\fml{S};\fml{M}_1,\mbf{v}_{1..m})) \\ \nonumber
&= \frac{1}{2} \times \frac{m-|\fml{S}|-2}{2^{m-|\fml{S}|}}
\end{align}
In this case, $\phi(\fml{S}\cup\{j\};\fml{M},\mbf{v}_{1..n}) - \phi(\fml{S};\fml{M},\mbf{v}_{1..n}) = -\frac{1}{4}$ if $|\fml{S}| = m-1$, which is its unique minimal value.
$\phi(\fml{S}\cup\{j\};\fml{M},\mbf{v}_{1..n}) - \phi(\fml{S};\fml{M},\mbf{v}_{1..n}) = 0$ if $|\fml{S}| = m-2$,
and $\phi(\fml{S}\cup\{j\};\fml{M},\mbf{v}_{1..n}) - \phi(\fml{S};\fml{M},\mbf{v}_{1..n}) > 0$ if $|\fml{S}| < m-2$.
\end{enumerate}

Next, we prove $|\sv(n)| > |\sv(j)|$ by showing
$\sv(n) + \sv(j) > 0$ and $\sv(n) - \sv(j) > 0$.
Note that $\sv(n) > 0$.
Additionally, $\phi(\fml{S}\cup\{j,n\};\fml{M},\mbf{v}_{1..n}) - \phi(\fml{S}\cup\{n\};\fml{M},\mbf{v}_{1..n}) < 0$
and $\phi(\fml{S}\cup\{j\};\fml{M},\mbf{v}_{1..n}) - \phi(\fml{S};\fml{M},\mbf{v}_{1..n}) < 0$
only when $|\fml{S}| = m-1$.
\begin{enumerate}
\item
For $\sv(n)$:
\begin{align}
\sv(n)
&= \sum_{\fml{S}\subseteq \fml{F}\setminus\{n\}}\frac{|\fml{S}|!(m-|\fml{S}|)!}{(m+1)!} \times 
(\phi(\fml{S}\cup\{n\};\fml{M},\mbf{v}_{1..n}) - \phi(\fml{S};\fml{M},\mbf{v}_{1..n})) \\ \nonumber
&= \sum_{\fml{S}\subseteq \fml{F}\setminus\{n\}}\frac{|\fml{S}|!(m-|\fml{S}|)!}{(m+1)!} \times \frac{1}{2} \phi(\fml{S};\fml{M}_1,\mbf{v}_{1..m}) \\ \nonumber
&= \frac{1}{2} \times \frac{1}{m+1} \times \sum_{\fml{S}\subseteq \fml{F}\setminus\{n\}}\frac{|\fml{S}|!(m-|\fml{S}|)!}{m!} \phi(\fml{S};\fml{M}_1,\mbf{v}_{1..m}) \\ \nonumber
&= \frac{1}{2} \times \frac{1}{m+1} \times \sum_{\fml{S}\subseteq \fml{F}\setminus\{n\}}\frac{|\fml{S}|!(m-|\fml{S}|)!}{m!} 
\times \frac{m-|\fml{S}|}{2^{m-|\fml{S}|}} \\ \nonumber
&= \frac{1}{2} \times \frac{1}{m+1} \times \sum_{0 \le |\fml{S}| \le m}\frac{|\fml{S}|!(m-|\fml{S}|)!}{m!} \times \frac{m!}{|\fml{S}|!(m-|\fml{S}|)!}
\times \frac{m-|\fml{S}|}{2^{m-|\fml{S}|}} \\ \nonumber
&= \frac{1}{2} \times \frac{1}{m+1} \times \sum_{0 \le |\fml{S}| \le m}\frac{m-|\fml{S}|}{2^{m-|\fml{S}|}} 
= \frac{1}{2} \times \frac{1}{m+1} \times \sum^{m}_{k=1}\frac{k}{2^{k}} \\ \nonumber
&= \frac{1}{2} \times \frac{1}{m+1} \times \frac{2^{m+1} - m - 2}{2^{m}} = \frac{1}{m+1} \times \frac{2^{m+1} - m - 2}{2^{m+1}} 
\end{align}
\item
For a feature $j \neq n$, consider the subset $\fml{S}=\fml{F}'\setminus\{j\}$ where $|\fml{S}| = m-1$ and the feature $n$:
\begin{align}
&\frac{|\fml{S}\cup\{n\}|!(m-|\fml{S}\cup\{n\}|)!}{(m+1)!} \times \frac{m-|\fml{S}|-2}{2^{m-|\fml{S}|}} \\ \nonumber
&= \frac{m!(m-m)!}{(m+1)!} \times \frac{m-(m-1)-2}{2^{m-(m-1)}} \\ \nonumber
&= -\frac{1}{2} \times \frac{1}{m+1}
\end{align}
\item
For a feature $j \neq n$, consider the subset $\fml{S}=\fml{F}'\setminus\{j\}$ where $|\fml{S}| = m-1$ and without the feature $n$:
\begin{align}
&\frac{|\fml{S}|!(m-|\fml{S}|)!}{(m+1)!} \times \frac{1}{2} \times \frac{m-|\fml{S}|-2}{2^{m-|\fml{S}|}} \\ \nonumber
&=  \frac{1}{2} \times \frac{(m-1)!(m-(m-1))!}{(m+1)!} \times \frac{m-(m-1)-2}{2^{m-(m-1)}} \\ \nonumber
&= -\frac{1}{4} \times \frac{1}{m(m+1)}
\end{align}
\end{enumerate}
We consider the sum of these three values:
\begin{align}
&\frac{1}{m+1} \times \frac{2^{m+1} - m - 2}{2^{m+1}} -\frac{1}{2} \times \frac{1}{m+1} -\frac{1}{4} \times \frac{1}{m(m+1)} \\ \nonumber
&= \frac{1}{m+1} \times \left( \frac{(2^{m+1} - m - 2)m}{m2^{m+1}} -\frac{m2^m}{m2^{m+1}} - \frac{2^{m-1}}{m2^{m+1}} \right) \\ \nonumber
&= \frac{1}{m(m+1)2^{m+1}} \times \left( (m - \frac{1}{2})2^m - m^2 - 2m \right)
\end{align}
Since $m \ge 3$, the sum of these three values is always greater than 0.
Thus, we can conclude that $\sv(n) + \sv(j) > 0$.

To show $\sv(n) - \sv(j) > 0$, we focus on all subsets $\fml{S}\subseteq\fml{F}'$ where $|\fml{S}| < m-2$.
This is because, as previously stated,
$\phi(\fml{S}\cup\{j,n\};\fml{M},\mbf{v}_{1..n}) - \phi(\fml{S}\cup\{n\};\fml{M},\mbf{v}_{1..n}) \le 0$
and $\phi(\fml{S}\cup\{j\};\fml{M},\mbf{v}_{1..n}) - \phi(\fml{S};\fml{M},\mbf{v}_{1..n}) \le 0$ if $|\fml{S}| \ge m-2$.

Moreover, for all subsets $\fml{S}\subseteq\fml{F}'$ where $|\fml{S}|=k$ where $0 < k \le m-3$, we compute the following three quantities:
\[
Q_1 := \sum_{\fml{S}\subseteq \fml{F}', |\fml{S}|=k} \phi(\fml{S}\cup\{n\};\fml{M},\mbf{v}_{1..n}) - \phi(\fml{S};\fml{M},\mbf{v}_{1..n})
\]
\[
Q_2 := \sum_{\fml{S}\subseteq \fml{F}'\setminus\{j\}, |\fml{S}|=k-1} \phi(\fml{S}\cup\{j,n\};\fml{M},\mbf{v}_{1..n}) - \phi(\fml{S}\cup\{n\};\fml{M},\mbf{v}_{1..n})
\]
\[
Q_3 := \sum_{\fml{S}\subseteq \fml{F}'\setminus\{j\}, |\fml{S}|=k} \phi(\fml{S}\cup\{j\};\fml{M},\mbf{v}_{1..n}) - \phi(\fml{S};\fml{M},\mbf{v}_{1..n})
\]
and show that $Q_1 - Q_2 - Q_3 > 0$.
Note that $Q_1$, $Q_2$ and $Q_3$ share the same coefficient $\frac{k!(n-k-1)!}{n!}$.
\begin{enumerate}
\item
For the feature $n$, we pick all possible subsets $\fml{S}\subseteq\fml{F}'$ where $|\fml{S}| = k$, which implies $|\fml{S} \cup \{n\}| = k+1$, then:
\[
Q_1 = \binom{m}{|\fml{S}|} \times \frac{1}{2} \times \frac{m-|\fml{S}|}{2^{m-|\fml{S}|}} 
= \binom{m}{k} \times \frac{1}{2} \times \frac{m-k}{2^{m-k}}
\]
\item
For a feature $j \neq n$ and consider the feature $n$, we pick all possible subsets $\fml{S}\subseteq\fml{F}'$ where $|\fml{S}| = k-1$, which implies $|\fml{S} \cup \{j, n\}| = k+1$, then:
\[
Q_2 = \binom{m-1}{|\fml{S}|} \times \frac{m-|\fml{S}|-2}{2^{m-|\fml{S}|}} 
= \binom{m-1}{k-1} \times \frac{m-(k-1)-2}{2^{m-(k-1)}}
= \binom{m-1}{k-1} \times \frac{1}{2} \times \frac{m-k-1}{2^{m-k}}
\]
\item
For a feature $j \neq n$, without considering the feature $n$, we pick all possible subsets $\fml{S}\subseteq\fml{F}'$ where $|\fml{S}| = k$, which implies $|\fml{S} \cup \{j\}| = k+1$, then:
\[
Q_3 = \binom{m-1}{|\fml{S}|} \times \frac{1}{2} \times \frac{m-|\fml{S}|-2}{2^{m-|\fml{S}|}} 
= \binom{m-1}{k} \times \frac{1}{2} \times \frac{m-k-2}{2^{m-k}}
\]
\end{enumerate}
Then we compute $Q_1 - Q_2 - Q_3$:
\begin{align}
&\binom{m}{k} \times \frac{1}{2} \times \frac{m-k}{2^{m-k}} - \binom{m-1}{k-1} \times \frac{1}{2} \times \frac{m-k-1}{2^{m-k}}
- \binom{m-1}{k} \times \frac{1}{2} \times \frac{m-k-2}{2^{m-k}} \\ \nonumber
&= \frac{1}{2} \times \frac{1}{2^{m-k}} 
\left[ \binom{m}{k} (m-k) - \binom{m-1}{k-1} (m - k - 1) - \binom{m-1}{k} (m - k - 2) \right] \\ \nonumber
&= \frac{1}{2} \times \frac{1}{2^{m-k}} 
\left[ \binom{m}{k} (m-k) - \binom{m-1}{k-1} (m - k) - \binom{m-1}{k} (m - k) + \binom{m-1}{k-1} + 2 \binom{m-1}{k} \right] \\ \nonumber
&= \frac{1}{2} \times \frac{1}{2^{m-k}} \left[ \binom{m-1}{k-1} + 2 \binom{m-1}{k} \right]
\end{align}
This means that $\sv(n) - \sv(j) > 0$.
Hence, we can conclude that $|\sv(n)| > |\sv(j)|$.
\qedhere
\end{proof}

\PropDisorderII*

\begin{proof}
Consider three classifiers $\fml{M}_1$, $\fml{M}_2$ and $\fml{M}_3$ implementing
non-constant boolean functions $\kappa_1$, $\kappa_2$ and $\kappa_3$, respectively.
Actually it is possible for $\kappa_1$ to be the constant function 0.
All of them are defined on the set of features $\fml{F}' = \{1, \dots, m\}$ where $m \ge 2$.
More importantly, $\kappa_1$, $\kappa_2$ and $\kappa_3$ satisfy the following conditions:
\begin{enumerate}
\item
$\kappa_2$ is a function predicting exactly one point $\mbf{v}_{1..m}$ to 1, for example, $\kappa_2$ can be $\bigland_{1 \le i \le m} \neg x_i$.
\item
For the point $\mbf{v}_{1..m}$ where $\kappa_2$ predicts 1, we have $\kappa_3(\mbf{v}_{1..m}) = 0$.
This implies $\kappa_2 \land \kappa_3 \models \bot$, that is, the conjunction of $\kappa_2$ and $\kappa_3$ is logically inconsistent.
\item
For any point $\mbf{x}_{1..m}$ such that $d_{H}(\mbf{x}_{1..m}, \mbf{v}_{1..m}) = 1$, 
where $d_{H}(\cdot)$ denotes the Hamming distance, we have $\kappa_3(\mbf{x}_{1..m}) = 1$.
\item
$\kappa_1 \land \kappa_2 \models \bot$ and $\kappa_1 \land \kappa_3 \models \bot$, 
indicating that the conjunction of $\kappa_1$ and $\kappa_2$ as well as the conjunction of $\kappa_1$ and $\kappa_3$ 
both equal to the constant function 0.
\item
$\kappa_1 \lor \kappa_2 \neq 1$ and $\kappa_1 \lor \kappa_3 \neq 1$,
indicating that neither the disjunction of $\kappa_1$ and $\kappa_2$ nor the disjunction of $\kappa_1$ and $\kappa_3$ 
equals the constant function 1.
\end{enumerate}
Let $\fml{F} = \fml{F}' \cup \{n-1, n\}$, we can build a new classifier $\fml{M}$ from $\fml{M}_1$, $\fml{M}_2$ and $\fml{M}_3$.
$\fml{M}$ is characterized by the boolean function defined as follows:
\begin{equation}
\kappa(\mbf{x}_{1..m},x_{n-1},x_{n}) :=
\left \{
\begin{array}{lcl}
\kappa_1(\mbf{x}_{1..m}) & \quad & \tn{if~$x_{n-1}=0$}\\[3pt]
\kappa_1(\mbf{x}_{1..m}) \lor \kappa_2(\mbf{x}_{1..m}) & \quad & \tn{if~$x_{n-1}=1 \land x_n = 0$}\\[3pt]
\kappa_1(\mbf{x}_{1..m}) \lor \kappa_3(\mbf{x}_{1..m}) & \quad & \tn{if~$x_{n-1} = 1 \land x_n = 1$}\\
\end{array}
\right.
\end{equation}
Besides, we can derive that
$(\neg x_n \land (\kappa_1 \lor \kappa_2)) \lor (x_n \land (\kappa_1 \lor \kappa_3)) = \kappa_1 \lor (\neg x_n \land \kappa_2) \lor (x_n \land \kappa_3)$.
So we have $\kappa_1 \models \kappa_1 \lor (\neg x_n \land \kappa_2) \lor (x_n \land \kappa_3)$.
Choose the $m$-dimensional point $\mbf{v}_{1..m}$ such that
$\kappa_1(\mbf{v}_{1..m}) = \kappa_3(\mbf{v}_{1..m}) = 0$ but $\kappa_2(\mbf{v}_{1..m}) = 1$.
Extend $\mbf{v}_{1..m}$ with $v_{n-1} = v_n=1$, 
let $\mbf{v}_{1..n} = (\mbf{v}_{1..m}, 1, 1)$ be the $n$-dimensional point, it follows that $\kappa(\mbf{v}_{1..n}) = 0$.
Based on the proof of \cref{prop:irr}, feature $n-1$ is irrelevant.
To prove that feature $n$ is relevant, we assume the contrary, i.e., that $n$ is irrelevant.
In this case, we pick the point $\mbf{v}' = (\mbf{v}_{1..m}, 1, 0)$ from the feature space where $\kappa_2(\mbf{v}_{1..m}) = 1$.
Clearly, for this point we have $\kappa(\mbf{v}') = 1$, leading to a contradiction.
Thus, feature $n$ is relevant.

In the following, we prove that $|\sv(n-1)| > |\sv(n)|$ by showing that $\sv(n-1) - \sv(n) > 0$ and $\sv(n-1) + \sv(n) > 0$.
To simplify the notations, we will use $\mbf{x}'$ to denote an arbitrary $n$-dimensional point $\mbf{x}_{1..n}$,
and $\mbf{y}$ to denote an arbitrary $m$-dimensional point $\mbf{x}_{1..m}$. 
For any subset $\fml{S} \subseteq \fml{F}'$, we now focus on feature $n-1$.
\begin{enumerate}
\item For the feature $n-1$, consider an arbitrary subset $\fml{S}\subseteq\fml{F}'$ and without the feature $n$, then:
\begin{align}
&\phi(\fml{S}\cup\{n-1\};\fml{M},\mbf{v}_{1..n}) - \phi(\fml{S};\fml{M},\mbf{v}_{1..n}) \\ \nonumber
&= \left( \frac{1}{2^{|\fml{F}'\setminus\fml{S}|+1}}\sum_{\mbf{x}'\in\Upsilon(\fml{S}\cup\{n-1\};\mbf{v}_{1..n})}\kappa(\mbf{x}') \right)
- \left( \frac{1}{2^{|\fml{F}'\setminus\fml{S}|+2}}\sum_{\mbf{x}'\in\Upsilon(\fml{S};\mbf{v}_{1..n})}\kappa(\mbf{x}') \right) \\ \nonumber
&= \frac{1}{2^{|\fml{F}'\setminus\fml{S}|+1}} \left( \sum_{\mbf{y}\in\Upsilon(\fml{S};\mbf{v}_{1..m})} (\kappa_1(\mbf{y}) \lor \kappa_3(\mbf{y}))
+ \sum_{\mbf{y}\in\Upsilon(\fml{S};\mbf{v}_{1..m})}(\kappa_1(\mbf{y}) \lor \kappa_2(\mbf{y})) \right) \\ \nonumber
&- \frac{1}{2^{|\fml{F}'\setminus\fml{S}|+2}} \left( \sum_{\mbf{y}\in\Upsilon(\fml{S};\mbf{v}_{1..m})} (\kappa_1(\mbf{y}) \lor \kappa_3(\mbf{y}))
+ \sum_{\mbf{y}\in\Upsilon(\fml{S};\mbf{v}_{1..m})}(\kappa_1(\mbf{y}) \lor \kappa_2(\mbf{y})) \right) \\ \nonumber
&- \frac{1}{2^{|\fml{F}'\setminus\fml{S}|+2}} \left( \sum_{\mbf{y}\in\Upsilon(\fml{S};\mbf{v}_{1..m})} \kappa_1(\mbf{y}) + 
\sum_{\mbf{y}\in\Upsilon(\fml{S};\mbf{v}_{1..m})} \kappa_1(\mbf{y}) \right) \\ \nonumber
&= \frac{1}{2^{|\fml{F}'\setminus\fml{S}|+2}} \left( \sum_{\mbf{y}\in\Upsilon(\fml{S};\mbf{v}_{1..m})} (\kappa_1(\mbf{y}) \lor \kappa_3(\mbf{y}))
+ \sum_{\mbf{y}\in\Upsilon(\fml{S};\mbf{v}_{1..m})}(\kappa_1(\mbf{y}) \lor \kappa_2(\mbf{y})) \right) \\ \nonumber
&- \frac{1}{2^{|\fml{F}'\setminus\fml{S}|+2}} \left( \sum_{\mbf{y}\in\Upsilon(\fml{S};\mbf{v}_{1..m})} \kappa_1(\mbf{y}) +
\sum_{\mbf{y}\in\Upsilon(\fml{S};\mbf{v}_{1..m})} \kappa_1(\mbf{y}) \right) \\ \nonumber
&= \frac{1}{2^{|\fml{F}'\setminus\fml{S}|+2}} \left( \sum_{\mbf{y}\in\Upsilon(\fml{S};\mbf{v}_{1..m})} \kappa_3(\mbf{y})
+ \sum_{\mbf{y}\in\Upsilon(\fml{S};\mbf{v}_{1..m})} \kappa_2(\mbf{y}) \right)
\end{align}
\item For the feature $n-1$, consider an arbitrary subset $\fml{S} \cup \{n\}$, then:
\begin{align}
&\phi(\fml{S}\cup\{n-1,n\};\fml{M},\mbf{v}_{1..n}) - \phi(\fml{S}\cup\{n\};\fml{M},\mbf{v}_{1..n}) \\ \nonumber
&= \left( \frac{1}{2^{|\fml{F}'\setminus\fml{S}|}}\sum_{\mbf{x}'\in\Upsilon(\fml{S}\cup\{n-1,n\};\mbf{v}_{1..n})}\kappa(\mbf{x}') \right)
- \left( \frac{1}{2^{|\fml{F}'\setminus\fml{S}|+1}}\sum_{\mbf{x}'\in\Upsilon(\fml{S}\cup\{n\};\mbf{v}_{1..n})}\kappa(\mbf{x}') \right) \\ \nonumber
&= \frac{1}{2^{|\fml{F}'\setminus\fml{S}|+1}} 
\left( \sum_{\mbf{y}\in\Upsilon(\fml{S};\mbf{v}_{1..m})} (\kappa_1(\mbf{y}) \lor \kappa_3(\mbf{y}))
- \sum_{\mbf{y}\in\Upsilon(\fml{S};\mbf{v}_{1..m})}\kappa_1(\mbf{y}) \right) \\ \nonumber
&= \frac{1}{2^{|\fml{F}'\setminus\fml{S}|+1}} \left( \sum_{\mbf{y}\in\Upsilon(\fml{S};\mbf{v}_{1..m})} \kappa_3(\mbf{y}) \right)
\end{align}
\end{enumerate}
Thus, we can conclude that $\sv(n-1) > 0$.
For any subset $\fml{S} \subseteq \fml{F}'$, we now focus on the feature $n$.
\begin{enumerate}
\item For the feature $n$, consider an arbitrary subset $\fml{S}\subseteq\fml{F}'$ and without the feature $n-1$, then:
\begin{align}
&\phi(\fml{S}\cup\{n\};\fml{M},\mbf{v}_{1..n}) - \phi(\fml{S};\fml{M},\mbf{v}_{1..n}) \\ \nonumber
&= \left( \frac{1}{2^{|\fml{F}'\setminus\fml{S}|+1}}\sum_{\mbf{x}'\in\Upsilon(\fml{S}\cup\{n\};\mbf{v}_{1..n})}\kappa(\mbf{x}') \right)
- \left( \frac{1}{2^{|\fml{F}'\setminus\fml{S}|+2}}\sum_{\mbf{x}'\in\Upsilon(\fml{S};\mbf{v}_{1..n})}\kappa(\mbf{x}') \right) \\ \nonumber
&= \frac{1}{2^{|\fml{F}'\setminus\fml{S}|+1}} \left( \sum_{\mbf{y}\in\Upsilon(\fml{S};\mbf{v}_{1..m})} (\kappa_1(\mbf{y}) \lor \kappa_3(\mbf{y}))
+ \sum_{\mbf{y}\in\Upsilon(\fml{S};\mbf{v}_{1..m})} \kappa_1(\mbf{y}) \right) \\ \nonumber
&- \frac{1}{2^{|\fml{F}'\setminus\fml{S}|+2}} \left( \sum_{\mbf{y}\in\Upsilon(\fml{S};\mbf{v}_{1..m})} (\kappa_1(\mbf{y}) \lor \kappa_3(\mbf{y}))
+ \sum_{\mbf{y}\in\Upsilon(\fml{S};\mbf{v}_{1..m})} \kappa_1(\mbf{y}) \right) \\ \nonumber
&- \frac{1}{2^{|\fml{F}'\setminus\fml{S}|+2}} \left( \sum_{\mbf{y}\in\Upsilon(\fml{S};\mbf{v}_{1..m})} (\kappa_1(\mbf{y}) \lor \kappa_2(\mbf{y})) + 
\sum_{\mbf{y}\in\Upsilon(\fml{S};\mbf{v}_{1..m})} \kappa_1(\mbf{y}) \right) \\ \nonumber
&= \frac{1}{2^{|\fml{F}'\setminus\fml{S}|+2}} 
\left( \sum_{\mbf{y}\in\Upsilon(\fml{S};\mbf{v}_{1..m})} (\kappa_1(\mbf{y}) \lor \kappa_3(\mbf{y}))
- \sum_{\mbf{y}\in\Upsilon(\fml{S};\mbf{v}_{1..m})}(\kappa_1(\mbf{y}) \lor \kappa_2(\mbf{y}))\right) \\ \nonumber
&= \frac{1}{2^{|\fml{F}'\setminus\fml{S}|+2}} 
\left( \sum_{\mbf{y}\in\Upsilon(\fml{S};\mbf{v}_{1..m})} \kappa_3(\mbf{y})
- \sum_{\mbf{y}\in\Upsilon(\fml{S};\mbf{v}_{1..m})} \kappa_2(\mbf{y}) \right)
\end{align}
\item For the feature $n$, consider an arbitrary subset $\fml{S} \cup \{n-1\}$, then:
\begin{align}
&\phi(\fml{S}\cup\{n-1,n\};\fml{M},\mbf{v}_{1..n}) - \phi(\fml{S}\cup\{n-1\};\fml{M},\mbf{v}_{1..n}) \\ \nonumber
&= \left( \frac{1}{2^{|\fml{F}'\setminus\fml{S}|}}\sum_{\mbf{x}'\in\Upsilon(\fml{S}\cup\{n-1,n\};\mbf{v}_{1..n})}\kappa(\mbf{x}') \right)
- \left( \frac{1}{2^{|\fml{F}'\setminus\fml{S}|+1}}\sum_{\mbf{x}'\in\Upsilon(\fml{S}\cup\{n-1\};\mbf{v}_{1..n})}\kappa(\mbf{x}') \right) \\ \nonumber
&= \frac{1}{2^{|\fml{F}'\setminus\fml{S}|+1}} 
\left( \sum_{\mbf{y}\in\Upsilon(\fml{S};\mbf{v}_{1..m})} (\kappa_1(\mbf{y}) \lor \kappa_3(\mbf{y}))
- \sum_{\mbf{y}\in\Upsilon(\fml{S};\mbf{v}_{1..m})} (\kappa_1(\mbf{y}) \lor \kappa_2(\mbf{y})) \right) \\ \nonumber
&= \frac{1}{2^{|\fml{F}'\setminus\fml{S}|+1}} 
\left( \sum_{\mbf{y}\in\Upsilon(\fml{S};\mbf{v}_{1..m})} \kappa_3(\mbf{y})
- \sum_{\mbf{y}\in\Upsilon(\fml{S};\mbf{v}_{1..m})} \kappa_2(\mbf{y}) \right)
\end{align}
\end{enumerate}
Note that $\phi(\fml{S}\cup\{n\};\fml{M},\mbf{v}_{1..n}) - \phi(\fml{S};\fml{M},\mbf{v}_{1..n}) < 0$
and $\phi(\fml{S}\cup\{n-1,n\};\fml{M},\mbf{v}_{1..n}) - \phi(\fml{S}\cup\{n-1\};\fml{M},\mbf{v}_{1..n}) < 0$ if and only if $\fml{S} = \fml{F}'$.
For any other proper subset $\fml{S} \subset \fml{F}'$, $\phi(\fml{S}\cup\{n\};\fml{M},\mbf{v}_{1..n}) - \phi(\fml{S};\fml{M},\mbf{v}_{1..n}) \ge 0$
and $\phi(\fml{S}\cup\{n-1,n\};\fml{M},\mbf{v}_{1..n}) - \phi(\fml{S}\cup\{n-1\};\fml{M},\mbf{v}_{1..n}) \ge 0$.

Moreover, for a fixed set $\fml{S} \subseteq \fml{F}'$, we have
$(\phi(\fml{S}\cup\{n-1\};\fml{M},\mbf{v}_{1..n}) - \phi(\fml{S};\fml{M},\mbf{v}_{1..n})) > 
(\phi(\fml{S}\cup\{n\};\fml{M},\mbf{v}_{1..n}) - \phi(\fml{S};\fml{M},\mbf{v}_{1..n}))$,
and $(\phi(\fml{S}\cup\{n-1,n\};\fml{M},\mbf{v}_{1..n}) - \phi(\fml{S}\cup\{n\};\fml{M},\mbf{v}_{1..n})) >
(\phi(\fml{S}\cup\{n-1,n\};\fml{M},\mbf{v}_{1..n}) - \phi(\fml{S}\cup\{n-1\};\fml{M},\mbf{v}_{1..n}))$.
Therefore, $\sv(n-1) > \sv(n)$.

In the following, we prove that $\sv(n-1) + \sv(n) > 0$ by focusing on all subsets $\fml{S} \subseteq \fml{F}$ where $m-2 \le |\fml{S}| \le m+1$.
For the feature $n-1$, we have:
\begin{align}
&\sum_{\substack{\fml{S}\subseteq \fml{F}\setminus\{n-1\} \\ m-2 \le |\fml{S}| \le m+1}}\frac{|\fml{S}|!(m+2-|\fml{S}|-1)!}{(m+2)!} \times 
(\phi(\fml{S}\cup\{n-1\};\fml{M},\mbf{v}_{1..n}) - \phi(\fml{S};\fml{M},\mbf{v}_{1..n})) \\ \nonumber
&= \sum_{|\fml{S}| = m+1, n \in \fml{S}}\frac{(m+1)!(m+1-(m+1))!}{(m+2)!}
\times \frac{1}{2^{m-m+1}} \times 0 \\ \nonumber
&+ \sum_{|\fml{S}| = m, n \in \fml{S}}\frac{m!(m+1-m)!}{(m+2)!}
\times \frac{1}{2^{m-(m-1)+1}} \times 1 \\ \nonumber
&+ \sum_{|\fml{S}| = m, n \not \in \fml{S}}\frac{m!(m+1-m)!}{(m+2)!}
\times \frac{1}{2^{m-m+2}} \times (0+1) \\ \nonumber
&+ \sum_{|\fml{S}| = m-1, n \in \fml{S}}\frac{(m-1)!(m+1-(m-1))!}{(m+2)!}
\times \frac{1}{2^{m-(m-2)+1}} \times 2 \\ \nonumber
&+ \sum_{|\fml{S}| = m-1, n \not \in \fml{S}}\frac{(m-1)!(m+1-(m-1))!}{(m+2)!}
\times \frac{1}{2^{m-(m-1)+2}} \times (1+1) \\ \nonumber
&+ \sum_{|\fml{S}| = m-2, n \not \in \fml{S}}\frac{(m-2)!(m+1-(m-2))!}{(m+2)!}
\times \frac{1}{2^{m-(m-2)+2}} \times (2+1) \\ \nonumber
%
%
&= \frac{8m+13}{16(m+2)(m+1)}
\end{align}
For the feature $n$, we have:
\begin{align}
&\sum_{\substack{\fml{S}\subseteq \fml{F}\setminus\{n\} \\ m-2 \le |\fml{S}| \le m+1}}\frac{|\fml{S}|!(m+2-|\fml{S}|-1)!}{(m+2)!} \times 
(\phi(\fml{S}\cup\{n\};\fml{M},\mbf{v}_{1..n}) - \phi(\fml{S};\fml{M},\mbf{v}_{1..n})) \\ \nonumber
&= \sum_{|\fml{S}| = m+1, n-1 \in \fml{S}}\frac{(m+1)!(m+1-(m+1))!}{(m+2)!}
\times \frac{1}{2^{m-m+1}} \times (0-1) \\ \nonumber
&+ \sum_{|\fml{S}| = m, n-1 \in \fml{S}}\frac{m!(m+1-m)!}{(m+2)!}
\times \frac{1}{2^{m-(m-1)+1}} \times (1-1) \\ \nonumber
&+ \sum_{|\fml{S}| = m, n-1 \not \in \fml{S}}\frac{m!(m+1-m)!}{(m+2)!}
\times \frac{1}{2^{m-m+2}} \times (0-1) \\ \nonumber
&+ \sum_{|\fml{S}| = m-1, n-1 \in \fml{S}}\frac{(m-1)!(m+1-(m-1))!}{(m+2)!}
\times \frac{1}{2^{m-(m-2)+1}} \times (2-1) \\ \nonumber
&+ \sum_{|\fml{S}| = m-1, n-1 \not \in \fml{S}}\frac{(m-1)!(m+1-(m-1))!}{(m+2)!}
\times \frac{1}{2^{m-(m-1)+2}} \times (1-1) \\ \nonumber
&+ \sum_{|\fml{S}| = m-2, n-1 \not \in \fml{S}}\frac{(m-2)!(m+1-(m-2))!}{(m+2)!}
\times \frac{1}{2^{m-(m-2)+2}} \times (2-1) \\ \nonumber
%
%
&= \frac{-6m-11}{16(m+2)(m+1)}
\end{align}
Their summation is $\frac{m+1}{8(m+2)(m+1)}$, since $m\ge2$,
$\sv(n-1) + \sv(n) > 0$.
Thus, it can be concluded that for the irrelevant feature $n-1$ and the relevant feature $n$, $|\sv(n-1)| > |\sv(n)|$.
\qedhere
\end{proof}

\begin{restatable}{corollary}{CorDisorderII}
    \label{cor:disorder2}
    For any $n\ge7$, there exist boolean functions defined on $n$
    variables, and at least one instance, for which there exists an
    irrelevant feature $i_1\in\fml{F}$, and a relevant feature $i_2\in\fml{F}\setminus\{i_1\}$,
    such that $\sv(i_1) > \sv(i_2) > 0$.
\end{restatable}

\begin{proof}
We utilize the function constructed in~\cref{prop:disorder2}, which is given by:
\begin{equation}
\kappa(\mbf{x}_{1..m},x_{n-1},x_{n}) :=
\left \{
\begin{array}{lcl}
\kappa_1(\mbf{x}_{1..m}) & \quad & \tn{if~$x_{n-1}=0$}\\[3pt]
\kappa_1(\mbf{x}_{1..m}) \lor \kappa_2(\mbf{x}_{1..m}) & \quad & \tn{if~$x_{n-1}=1 \land x_n = 0$}\\[3pt]
\kappa_1(\mbf{x}_{1..m}) \lor \kappa_3(\mbf{x}_{1..m}) & \quad & \tn{if~$x_{n-1} = 1 \land x_n = 1$}\\
\end{array}
\right.
\end{equation}
However, we choose a different function $\kappa_3$ that satisfies the following condition:
for any point $\mbf{x}_{1..m}$ such that $d_{H}(\mbf{x}_{1..m}, \mbf{v}_{1..m}) \le 2$,
where $d_{H}(\cdot)$ represents the Hamming distance,
$\kappa_3(\mbf{x}_{1..m}) = 1$,
According to the proof of~\cref{prop:disorder2}, it can be derived that $\sv(n-1) > 0$ and $\sv(n-1) > \sv(n)$.
In the following, we prove that $\sv(n) > 0$ by focusing on all subsets $\fml{S} \subseteq \fml{F}\setminus\{n\}$ where $m-4 \le |\fml{S}| \le m+1$,
and show that the sum of their values is greater than 0, which implies that $\sv(n) > 0$ when considering all possible subsets $\fml{S}$.
\begin{align}
&\sum_{\substack{\fml{S}\subseteq\fml{F}\setminus\{n\} \\ m-4 \le |\fml{S}| \le m+1}}\frac{|\fml{S}|!(m+2-|\fml{S}|-1)!}{(m+2)!} \times 
(\phi(\fml{S}\cup\{n\};\fml{M},\mbf{v}_{1..n}) - \phi(\fml{S};\fml{M},\mbf{v}_{1..n})) \\ \nonumber
&= \sum_{\substack{|\fml{S}| = m+1, n-1 \in \fml{S}}}\frac{(m+1)!(m+1-(m+1))!}{(m+2)!}
\times \frac{1}{2^{m-m+1}} \times (0-1) \\ \nonumber
&+ \sum_{|\fml{S}| = m, n-1 \in \fml{S}}\frac{m!(m+1-m)!}{(m+2)!}
\times \frac{1}{2^{m-(m-1)+1}} \times (1-1) \\ \nonumber
&+ \sum_{|\fml{S}| = m, n-1 \not \in \fml{S}}\frac{m!(m+1-m)!}{(m+2)!}
\times \frac{1}{2^{m-m+2}} \times (0-1) \\ \nonumber
&+ \sum_{|\fml{S}| = m-1, n-1 \in \fml{S}}\frac{(m-1)!(m+1-(m-1))!}{(m+2)!}
\times \frac{1}{2^{m-(m-2)+1}} \times (\binom{2}{1} + \binom{2}{2}-1) \\ \nonumber
&+ \sum_{|\fml{S}| = m-1, n-1 \not \in \fml{S}}\frac{(m-1)!(m+1-(m-1))!}{(m+2)!}
\times \frac{1}{2^{m-(m-1)+2}} \times (1-1) \\ \nonumber
&+ \sum_{|\fml{S}| = m-2, n-1 \in \fml{S}}\frac{(m-2)!(m+1-(m-2))!}{(m+2)!}
\times \frac{1}{2^{m-(m-3)+1}} \times (\binom{3}{1} + \binom{3}{2}-1) \\ \nonumber
&+ \sum_{|\fml{S}| = m-2, n-1 \not \in \fml{S}}\frac{(m-2)!(m+1-(m-2))!}{(m+2)!}
\times \frac{1}{2^{m-(m-2)+2}} \times (\binom{2}{1} + \binom{2}{2}-1) \\ \nonumber
&+ \sum_{|\fml{S}| = m-3, n-1 \in \fml{S}}\frac{(m-3)!(m+1-(m-3))!}{(m+2)!}
\times \frac{1}{2^{m-(m-4)+1}} \times (\binom{4}{1} + \binom{4}{2} - 1) \\ \nonumber
&+ \sum_{|\fml{S}| = m-3, n-1 \not \in \fml{S}}\frac{(m-3)!(m+1-(m-3))!}{(m+2)!}
\times \frac{1}{2^{m-(m-3)+2}} \times (\binom{3}{1} + \binom{3}{2}-1) \\ \nonumber
%
%
&= \frac{11m-47}{32(m+2)(m+1)}
\end{align}
Since $m \ge 5$, for all subsets $\fml{S} \subseteq \fml{F}\setminus\{n\}$ where $m-4 \le |\fml{S}| \le m+1$, their summation is greater than 0.
This implies that $\sv(n) > 0$.
Thus, it can be concluded that for the irrelevant feature $n-1$ and the relevant feature $n$, we have $\sv(n-1) > \sv(n) > 0$.
\end{proof}

\PropLBs*

\begin{proof}[Sketch]
    For~\cref{prop:irr}, there exist $2^{2^{(n-1)}}-n-3$ distinct non-constant functions $\kappa_2$. 
    For each such function $\kappa_2$, $\kappa_1$ can be defined by changing the prediction of some points predicted as 1 by $\kappa_2$ to 0. 
    It is evident that $\kappa_1 \models \kappa_2$ but $\kappa_1 \neq \kappa_2$.

    For~\cref{prop:rel,prop:disorder}, there exist $2^{2^{\sfrac{(n-1)}{2}}}-2$ distinct non-constant functions $\kappa_1$. 
    We can then define $\kappa_2$ by renaming each variable $x_i$ of $\kappa_1$ with a new variable $x_{m+i}$.

    For~\cref{prop:disorder2},
    the functions $\kappa_2$ and $\kappa_3$ are assumed to be fixed,
    while the flexibility lies in the choice of $\kappa_1$ ($\kappa_1$ can be 0 but cannot be 1).
    As $\kappa_2$ covers 1 point and $\kappa_3$ covers $n-2$ points,
    the remaining points in the feature space can be used to define the function $\kappa_1$.
    Thus, there are $2^{2^{n-2}-(n-2)-1}-1$ possible functions for $\kappa_1$. 
    \qedhere
\end{proof}

\end{document}